\documentclass[runningheads]{llncs}
\usepackage[T1]{fontenc}
\usepackage{graphicx}
\usepackage{booktabs}
\usepackage[misc]{ifsym}
\newcommand{\corr}{(\Letter)}
\usepackage{hyperref}       
\usepackage{url}            
\usepackage{amsfonts}       

\usepackage{microtype}      
\usepackage{xcolor}         
\usepackage{notation}
\usepackage{tabularx}
\usepackage{multirow}
\usepackage{algorithm}
\usepackage{subfig}
\usepackage{float}
\usepackage{mathtools}

\usepackage[square,sort,comma,numbers]{natbib}
\usepackage{xr}

\makeatletter
\newcommand*{\addFileDependency}[1]{
\typeout{(#1)}
\@addtofilelist{#1}
\IfFileExists{#1}{}{\typeout{No file #1.}}
}
\makeatother

\begin{document}

\title{Improving Discriminator Guidance in Diffusion Models}

\titlerunning{Improving Discriminator Guidance in Diffusion Models}

\author{%
    Alexandre Verine \inst{1} \corr \and
    Ahmed Mehdi Inane \inst{2} \and
    Florian Le Bronnec \inst{3} \and
    Benjamin Negrevergne \inst{3} \and
    Yann Chevaleyre \inst{3}
}
\tocauthor{%
    Alexandre Verine, Ahmed Mehdi Inane, Florian Le Bronnec, Benjamin Negrevergne, Yann Chevaleyre
}
\toctitle{Improving Discriminator Guidance in Diffusion Models}


\authorrunning{A. Verine et al.}

\institute{
    École Normale Supérieure Paris, PSL University, Paris, France \email{alexandre.verine@ens.fr}
    \and
    Mila, Quebec AI Institute, Université de Montréal, Montreal, Canada \email{mehdi-inane.ahmed@mila.quebec}
    \and
    LAMSADE, CNRS, Université Paris-Dauphine-PSL, Paris, France \email{\{florian.le-bronnec,benjamin.negrevergne,yann.chevaleyre\}@dauphine.psl.eu}
}


\maketitle              

\begin{abstract}
    Discriminator Guidance has become a popular method for efficiently refining pre-trained Score-Matching Diffusion models. However, in this paper, we demonstrate that the standard implementation of this technique does not necessarily lead to a distribution closer to the real data distribution.  Specifically, we show that training the discriminator using Cross-Entropy loss, as commonly done, can in fact increase the Kullback-Leibler divergence between the model and target distributions, particularly when the discriminator overfits. To address this, we propose a theoretically sound training objective for discriminator guidance that properly minimizes the KL divergence. We analyze its properties and demonstrate empirically across multiple datasets that our proposed method consistently improves over the conventional method by producing samples of higher quality.\footnote{Code: \url{https://github.com/AlexVerine/BoostDM} and Supplementary Materials \url{https://arxiv.org/abs/2503.16117}}
\end{abstract}

\section{Introduction}

Diffusion based generative models have proven to be effective in numerous fields, including image and video generation \citep{dhariwal2021diffusion,kim2024consistency,ho2022imagen,ho2022video}, graphs \citep{liu2023generative}, and audio synthesis \citep{kong2021diffwave}, among others. Diffusion models are trained to iteratively denoise samples from a noise distribution to approximate the target data distribution. They have gained success in the generative modeling community for their ability to generate both high-quality samples, and to cover well the estimated distribution. However, this comes with a high computational cost both for sampling new data (which requires multiple denoising steps) and for training (which requires training the model multiple time for each level of denoising).

\begin{figure*}[!t]
    \centering
    \includegraphics[width=\textwidth]{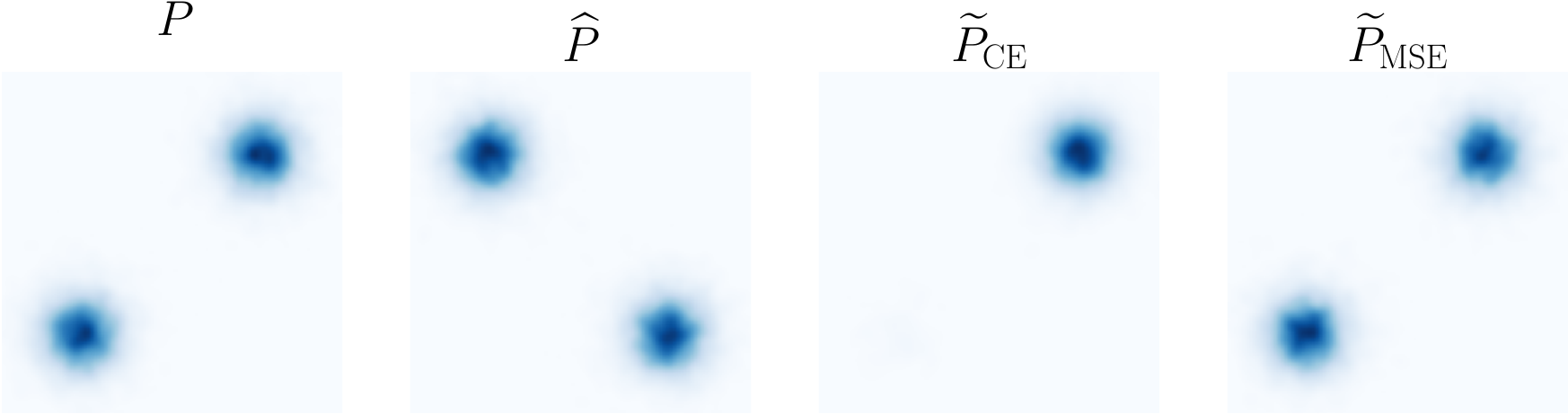}
    \caption{Illustration of generating samples by a reverse diffusion process to sample from the target distribution $P$ using true score function $\nabla \log p_t$, from the learned distribution $\whP$ using the learned score function $\vs_\theta$, from a poor approximation $\wtP_{\mathrm{CE}}$ using a refined score $\vs_\theta + \nabla d_\phi$ with low cross-entropy $\calL_{\mathrm{CE}^d(\phi)}$ and, finally, from a good approximation $\wtP_{\mathrm{MSE}}$ using a refined score $\vs_\theta +\nabla d_\phi$ with low loss $\calL_{\mathrm{MSE}^d(\phi)}$ introduced in Section~\ref{sec:improvement}. }
    \label{fig:2Dplots}
\end{figure*}

Consequently, several approaches focus on post-training strategies to refine a pre-trained model’s distribution at a lower computational cost \citep{song_score-based_2021,dhariwal_diffusion_2021,xie2023difffit,Lu_2023_CVPR}. Among these, Discriminator Guidance (DG) \citep{kim_refining_2023} has recently emerged as a promising refinement method, demonstrating strong empirical results \citep{kim_refining_2023,kelvinius2023discriminator,tsonis2023mitigating}. This approach refines the generation process by training a discriminator to distinguish generated samples from real samples at different diffusion steps. The discriminator is then used to estimate the log density ratio between the real data distribution \(P\) and the learned distribution \(\widehat{P}\), with the gradient of this estimate acting as a correction term during generation.

However, this introduces a fundamental discrepancy between training and inference: while the discriminator is trained to approximate the density ratio, inference relies on its gradient, which is \textbf{not necessarily a reliable estimate of the target gradient}. As a result, refining the pre-trained model in this way can degrade generation quality.

In this work, we make the following contributions:
\begin{itemize}
    \item We formally demonstrate in Theorem~\ref{thm:suboptimal} that training a discriminator to minimize Cross-Entropy and using it for refinement can lead to arbitrarily poor results in the final distribution.
    \item In Theorem~\ref{thm:overfitting}, we show that overfitting the discriminator is a sufficient condition for the refined distribution to deteriorate in terms of KL divergence compared to the pre-trained model.
    \item We propose a reformulation of the DG objective, introducing a new optimization criterion that directly improves the refined distribution by leveraging the gradient of the log-likelihood ratio rather than its value (its benefits over standard Cross-Entropy minimization are illustrated on Figure~\ref{fig:2Dplots}).
    \item Finally, we demonstrate that our method enhances sample quality in diffusion models across benchmark image generation datasets, including CIFAR-10, FFHQ, and AFHQ-v2.
\end{itemize}

By providing a theoretically grounded alternative to the standard DG objective, our approach improves both the understanding and effectiveness of discriminator-guided refinement in generative modeling.

\paragraph{Notations:} We denote $P$, $\whP$ and $\wtP$ the target, the learned and the refined distributions defined on $\reals^d$. Their densities are denoted $p$, $\whp$ and $\wtp$. We will denote with an indexed $t$ the diffused distribution $P_t$, $\whP_t$ and $\wtP_t$ at time $t$ and their densities $p_t$, $\whp_t$ and $\wtp_t$. We denote by $\KL$ the Kullback-Leibler divergence.
\section{Related Works}\label{sec:relatedworks}
In a trained diffusion model, the divergence between the learned distribution $\whP$ and the target distribution $P$ arises from two primary sources of error: \textit{sampling} errors, which stem from the discretization scheme used to solve the stochastic differential equation for sample generation, and \textit{estimation} errors, which originate from the discrepancy between the model’s final estimate of the score, $\nabla\log\whp_t$ and the target score $\nabla\log p_t$. This section reviews existing approaches that aim to mitigate this discrepancy by addressing one or both sources of error:
\begin{itemize}
    \item \textbf{Sampling strategies:} A first line of work attempts to correct the sampling error due to the discretization of the backward SDE (Equation ~\eqref{diffprocrev}). Since solving this equation numerically is fundamental to sample generation in diffusion models, the choice of the discretization scheme significantly impacts sample quality. Traditional solvers, such as the Euler-Maruyama method \citep{kloeden1992stochastic}, are widely used but can introduce bias or require a large number of function evaluations to achieve high fidelity. To address these limitations, researchers have explored improved numerical schemes specifically tailored to score-based generative modeling. For example, \citet{jolicoeur-martineau_gotta_2021} propose an adaptive step size strategy for the SDE solver, allowing for more efficient sample generation by dynamically adjusting the discretization step. This approach reduces numerical error and enables high-quality sample synthesis in fewer iterations compared to fixed-step solvers. \citet{xu_restart_2023} proposed the Restart sampling algorithm, which alternates between adding noise through additional forward steps and strictly following a backward ordinary differential equation (ODE). This approach balances discretization errors and the contraction property of stochasticity, resulting in accelerated sampling speeds while maintaining or enhancing sample quality.
    \item \textbf{Correcting the estimation:} Another line of work attempts to correct the estimated score function with additional information from auxiliary models. For instance, the classifier-guided approach refines score estimation by incorporating gradients from an external classifier trained to predict class probabilities. These gradients are used to adjust the predicted score of the diffusion model, steering the sampling process toward more semantically meaningful outputs \citep{dhariwal_diffusion_2021}. Since the estimation error term is given by the gradient of the log density ratio between the target distribution $P$ and model distribution $\whP$, \citet{kim_refining_2023} introduced \textit{discriminator guidance}, a method leveraging a traditional density estimation technique by training a discriminator $d_\phi$. The gradient of the output of $d_\phi$ is then added to the score estimate during sampling, which provably reduces the discrepancy between $P$ and $\whP$. This method builds on a broad body of literature that leverages discriminators to estimate the log-density ratio for enhancing the generation process of Generative Adversarial Networks \cite{azadi_discriminator_2019,ansari_refining_2021,verine_precision-recall_2023,verine_optimal_2024,che_your_2020}.Recent works have proposed complementary algorithmic improvements to this approach: \citet{kelvinius2023discriminator} propose a sequential Monte-Carlo based algorithm to correct the estimation errors of the discriminator for autoregressive diffusion models, and \citet{tsonis2023mitigating} propose a hybrid algorithm that merges discriminator guidance and a scaling factor to mitigate the exposure bias induced by the training process of diffusion models. Our work focuses on improving the training of the discriminator itself, and can be integrated with all the aforementioned methods to further improve sample quality and model performance.
\end{itemize}

\section{ Background on Score-matching Diffusion Models}\label{sec:background}
\subsection{Diffusion Process}
Let $\{\vx_t\}_{t\in[0, T]}$ be a diffusion process defined by an \^Ito SDE:
\begin{align}\label{diffproc}
    \dx_t= \vf(\vx_t, t)\dt + g(t)\dw,
\end{align}
where $\vf(., t)\in \reals^d\to\reals^d$ is the drift, $g(t):\reals\to \reals$ is the diffusion coefficient and $\vw\in \reals^d$ is a standard Wiener process. It defines a sequence of distributions $\left\{P_t\right\}_{t\in[0, T]}$, with densities $p_t$. With this definition, the target distribution is $P=P_0$ and $\vf$, $g$ and $T$ are chosen such that $P_T$ tends toward a tractable distribution $Q$ with density $q$. In practice, $Q$ is a normal distribution in $\reals^d$. Using the time-reversed diffusion process \cite{anderson_reverse-time_1982} of Equation~\eqref{diffproc} we can define a generation scheme to sample from the target distribution $P$.  We have the reverse SDE :
\begin{align}\label{diffprocrev}
    \dx_t = \left[f(\vx_t, t) - g(t)^2\nabla_{\vx_t}\log p_t (\vx_t)\right]\dt +  g(t)\d \bar{\vw},
\end{align}
where $\d \bar \vw$ denotes a different standard Wiener process.
\subsection{Score-Based Generative Models}
Taking advantage of the reverse SDE in Equation~\eqref{diffprocrev}, we can train generative models based on scores by training a neural network $\vs_\theta(\vx_t, t)$ to estimate the value of the score $\nabla_{\vx_t} \log p_t(\vx_t)$. To do so, we train U-Net  \citet{ronneberger_u-net_2015}, a function $\reals^d\times\reals \to \reals^d$, to minimize the \emph{score matching loss} \citep{hyvarinen_estimation_2005}:
\begin{align}\label{eq:scorematching}
    \calL_{\mathrm{SM}}^{\vs}(\theta) = \frac{1}{2}\int_0^T\lambda(t)
    \E_{P_t}\left[\Vert \nabla_{\vx_t} \log p_t(\vx_t) - \vs_\theta(\vx_t, t)\Vert_2^2 \right]\dt.
\end{align}
Note that the weighting function $\lambda:[0,T]\to]0,+\infty[$ depends on the type of SDE. Although the score matching loss is not directly computable, there exist methods to estimate it using Sliced Score Matching \citep{song_sliced_2019}. \citet{vincent_connection_2011} shows that the score matching loss is equivalent, up to a constant additive term independent of $\theta$, to the \emph{denoising score matching} loss:
\begin{align}\label{eq:score_matching_loss}
    \calL_{\mathrm{MSE}}^{\vs}(\theta) = \int_0^T\lambda(t)   \E_{P_0, P_{t\vert\vx_0}}\left[\Vert \nabla_{\vx_t} \log p_t(\vx_t\vert \vx_0) - \vs_\theta(\vx_t, t)\Vert_2^2 \right]\dt,
\end{align}
where $P_{t\vert\vx_0}$ is conditional distribution of $\vx_t$ given $\vx_0$. This distribution is typically chosen to be a Gaussian distribution with a mean depending on $\vx_0$ and $t$ and a variance depending on $t$. Therefore, this loss is based on the mean square error of the denoised reconstruction, which is widely used in practice as it is easier to compute and optimize \citep{karras_elucidating_2022}. The learned score function defines a new reverse diffusion process:
\begin{align}\label{diffprocrevlearned}
    \dx_t = \left[\vf(\vx_t, t) - g(t)^2\vs_\theta(\vx_t, t)\right]\dt +  g(t)\d \bar{\vw}.
\end{align}
It defines  learned distributions $\left\{\whP_t\right\}_{t\in[0,T]}$ with densities $\whp_t$ such that:
\begin{align}
    \begin{cases}
        \whp_T(\vx_T) = q(\vx_T), \\
        \nabla_{\vx_t} \log \whp_t (\vx_t, t) =\vs_\theta(\vx_t, t).
    \end{cases}
\end{align}
In general, the induced distribution $\whP_0=\whP$ does not perfectly match the target distribution. \citet{song_maximum_2021} shows that under the right assumptions (detailed in Appendix~\ref{app:sec:assumptionssong}), the Kullback-Leibler divergence is given by:
\begin{align}\label{eq:klphat}
    \KL(P\Vert\whP)  & = \KL(P_T\Vert Q)
    +    \frac{1}{2} & \int_0^T g^2(t)\E_{P_t}\left[\Vert \nabla_{\vx_t} \log p_t(\vx_t) - \vs_\theta(\vx_t, t)\Vert_2^2 \right]\dt.
\end{align}
In practice, the score function $\vs_\theta$ is learned using a U-Net architecture \citep{ronneberger_u-net_2015} trained to minimize $\cal L_{\mathrm{MSE}}^{\vs}$ and thus helps to reduce the dissimilarity between the learned distribution $\whP$ and the target distribution $P$.

\subsection{Discriminator Guided Diffusion}
To enhance the generative process and minimize the discrepancy between the generated distribution $\widehat{P}$ and the target distribution $P$, \citet{kim_refining_2023} propose leveraging the density ratio $p_t(\vx_t)/\widehat{p}_t(\vx_t)$. By incorporating this density ratio, the score estimation can be refined using the following identity:
\begin{align}
    \nabla_{\vx_t} \log p_t(\vx_t) = \nabla_{\vx_t} \log  \whp(\vx_t)+\nabla_{\vx_t} \log p_t(\vx_t)/\whp_t(\vx_t).
\end{align}
However, the density ratio cannot be computed directly. To address this, the authors propose training a discriminator to approximate it. In practice, the discriminator is implemented as a neural network $\vd_\phi(\vx_t, t)$, which is trained to minimize the cross-entropy (CE) loss between real and generated data:
\begin{align}\label{eq:CE}
    \begin{split}
        \calL_{\mathrm{CE}}^{d}(\phi) = \int_0^T\lambda(t)\Big[  \E_{P_t}\left[-\log \sigma\left(\vd_\phi(\vx_t, t)\right)\right]
            +                                                        \E_{\whP_t}\left[-\log\left(1- \sigma\left(\vd_\phi(\vx_t, t)\right)\right)\right] \Big]\d t,
    \end{split}
\end{align} where $\sigma$ is the sigmoid function.If the discriminator $\vd_{\phi}(.,t)$ were capable of representing any measurable function from $\mathbb{R}^d$ to $\mathbb{R}$, then the optimal discriminator could be used to compute the density ratio \citep{sugiyama_density_2010, nowozin_f-gan_2016, nguyen_surrogate_2009}:
\begin{align}
    r\s(\vx_t, t) = p_t(\vx_t)/\whp_t(\vx_t)=e^{\vd_{\phi\s}(\vx_t, t)}.
\end{align}
However, in practice, the expressivity of the discriminator is limited, and thus the discriminator does not perfectly estimate the density ratio, and the estimated density ratio is defined as $r_{\phi}(\vx_t, t) = \exp(\vd_\phi(\vx_t, t))$. Using this estimation, the score refinement can be computed as$\nabla_{\vx_t} \log r_{\phi}(\vx_t, t) = \nabla_{\vx_t} \vd_{\phi}(\vx_t, t)$. The reverse diffusion process using the discriminator guidance defines a sequence of refined distributions $\left\{\wtP_t\right\}_{t\in[0,T]}$, with densities $\wtp_t$ such that:
\begin{align}
    \begin{cases}
        \wtp_T(\vx_T) = \pi(\vx_T), \\
        \nabla_{\vx_t} \log \wtp_t (\vx_t, t) =\vs_\theta(\vx_t, t) +\nabla_{\vx_t}\vd_\phi(\vx_t, t)\coloneq\widetilde{\vs}_{\theta,\phi}(\vx_t, t).
    \end{cases}
\end{align}
Similarly to the classical generative process, $\wtP = \wtP_0$ does not perfectly match the target distribution. The Kullback-Leibler divergence between the refined distribution and the target distribution can be computed as follows applying the same assumptions as Equation~\eqref{eq:klphat}.
\begin{align}\label{eq:klptilde}
    \KL(P \Vert \wtP) =  \KL(P_T \Vert Q) + \frac{1}{2}\int_0^T g(t)^2 \E_{P_t}\bigg[\Vert \nabla_{\vx_t} \log p_t(\vx_t)  - \widetilde{\vs}_{\theta,\phi}(\vx_t, t)\Vert_2^2 \bigg]\dt.
\end{align}
In Equation~\eqref{eq:klptilde}, we note that the dissimilarity between the target distribution and the refined distribution depends on the MSE between the difference in scores and the discriminator \textit{gradient}. However, the model is trained to minimize CE and there is no guaranty that the CE is the optimal loss for the refinement.

\section{Misalignment between the cross-entropy and the Kullback-Leibler divergence}\label{sec:optimality}

The DG framework \citep{kim_refining_2023} approximates the density ratio $p_t/\whp_t$ by training a discriminator with the CE and using its gradient for refinement. In this section, we formally show that minimizing cross-entropy does not necessarily improve the refined distribution. Furthermore, we establish that this issue is not limited to pathological cases but naturally arises in the common overfitting regime, where the refined distribution deteriorates.

\subsubsection{Well-trained discriminator with poorly refined distribution:}
Our first result, stated in Theorem~\ref{thm:suboptimal}, shows that it is possible to construct a discriminator with an arbitrarily low CE, while the refined distribution is arbitrarily far from the target:\begin{theorem}\label{thm:suboptimal}
    Let $\left\{\vx(t)\right\}_{t\in[0, T]}$ be a diffusion process defined by Equation~\eqref{diffproc}. Assume that $\nabla \log \whp_t= \vs_\theta$ and $\nabla \log \wtp_t =\vs_\theta +\nabla d_\phi $ and that the induced distribution $P$, $\whP$, and $\wtP$ satisfies the assumptions detailed in Appendix~\ref{app:sec:assumptionssong}. Then, for every $\varepsilon>0$ and for every $\delta>0$, there exists a discriminator $\vd:\reals^d\times\reals\to\reals$ trained to minimize the cross-entropy such that:
    \begin{equation}
        \calL_{\mathrm{CE}}^{d}(\phi)\leq \varepsilon \et \KL(P\Vert \widetilde{P}) \geq \delta,
    \end{equation}
    where $\widetilde{P}$ is the distribution induced by discriminator guidance with $\vd$.
\end{theorem}
\begin{hproof}The detailed proof of Theorem~\ref{thm:suboptimal} is provided in Appendix~\ref{app:sec:proofsuboptimal}. The key insight is that the CE evaluates the discriminator's values \( d_{\phi}(\vx, t) \) but not its gradient \( \nabla_{\vx} d_{\phi}(\vx, t) \), which affects the generation process. The main argument is that a learned discriminator \( d_{\phi} \) oscillating around the optimal discriminator \( d^* \) can still achieve a low CE (Figure~\ref{fig:thm1}). Specifically, the magnitude of these oscillations determines how low the CE is, while their frequency degrades the approximation of \( \nabla_{\vx} d_{\phi}(\vx, t) \), leading to an increase in $\KL(P\Vert \widetilde{P})$.
\end{hproof}
Theorem~\ref{thm:suboptimal} establishes that minimizing cross-entropy does not necessarily
yield a better-refined distribution. Theorem~\ref{thm:overfitting} further demonstrates that this issue is not limited to rare or pathological cases but naturally emerges in the overfitting regime, a common occurrence in practical settings. As the discriminator memorizes the training data, its learned function develops high-frequency oscillations:
\begin{figure}[b!]
    \subfloat[High gradient MSE despite low cross-entropy (Theorem~\ref{thm:suboptimal})]{%
        \includegraphics[width=0.3\textwidth]{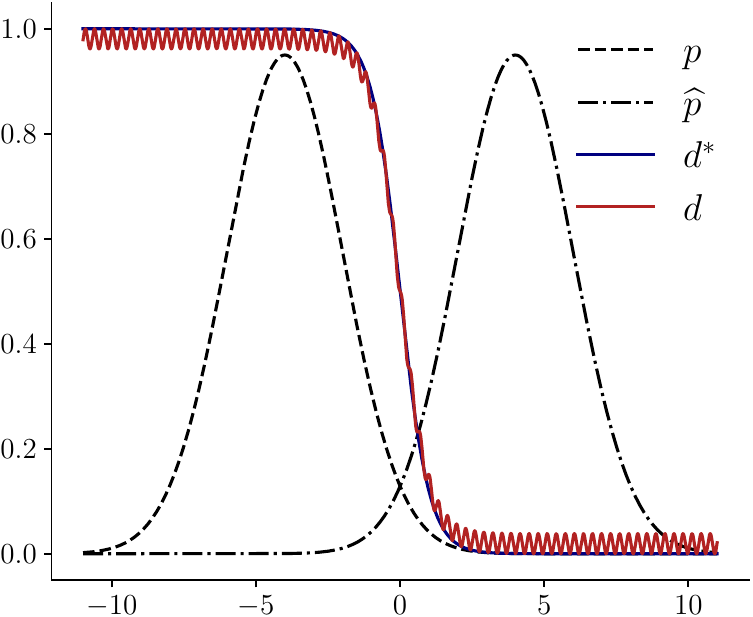}\label{fig:thm1}}\hfill
    \subfloat[Low gradient MSE due to small distribution overlap (Theorem~\ref{thm:overfitting})]{%
        \includegraphics[width=0.3\textwidth]{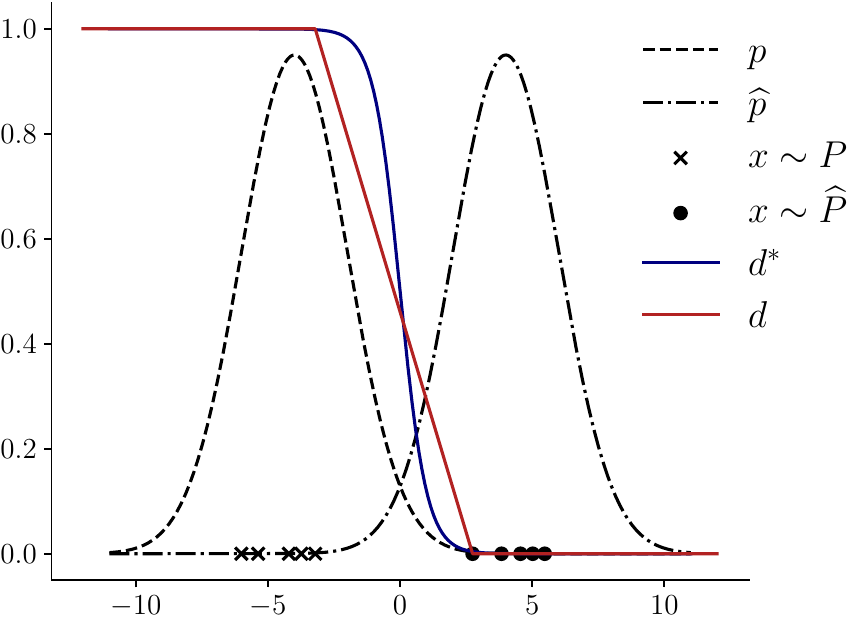}\label{fig:thm2a}}\hfill
    \subfloat[High gradient MSE from pathological overfitting (Theorem~\ref{thm:overfitting})]{%
        \includegraphics[width=0.3\textwidth]{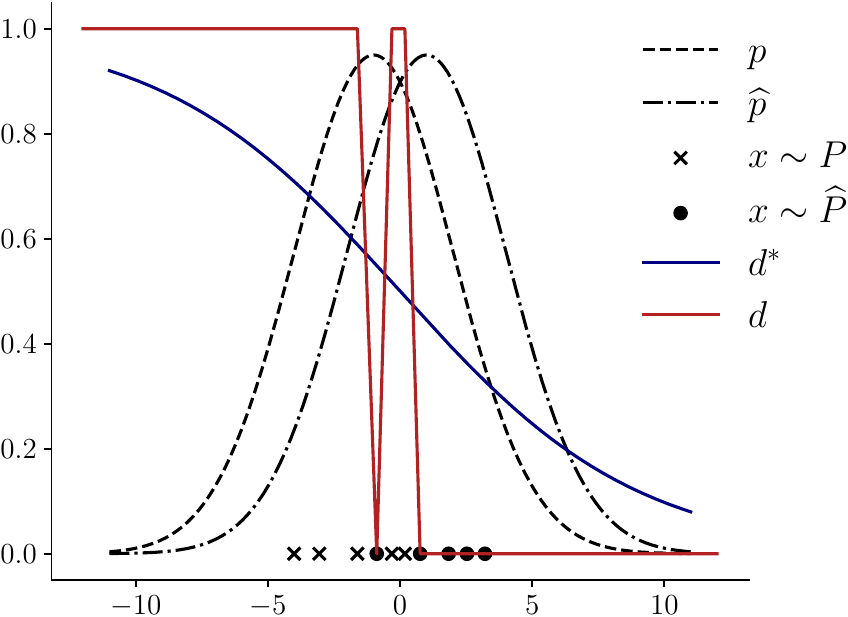}\label{fig:thm2b}}
    \caption{
    Illustration of pathological cases from Theorem~\ref{thm:suboptimal} and Theorem~\ref{thm:overfitting}.
    (Left) The cross-entropy loss \( \mathcal{L}_{\text{CE}}^d(\phi)\) is low, yet the MSE loss \( \mathcal{L}_{\text{MSE}}^d(\phi)\) is high due to substantial gradient mismatch.
    (Middle) Low cross-entropy loss with low MSE, as distributions minimally overlap.
    (Right) Despite low cross-entropy loss, the MSE loss diverges due to significant overlap, highlighting pathological overfitting.
    }
\end{figure}
\begin{theorem}\label{thm:overfitting}
    Let $P$ and $\widehat{P}$ be distributions on $\mathbb{R}$ with intersecting supports, admitting $L$-Lipschitz densities $p$ and $\hat{p}$. Let $x_{1},\dots,x_{N}\sim P^{N}$ and
    $x'_{1},\dots,x'_{N}\sim \widehat{P}^{N}$. Assume that there exists $\epsilon>0$ such that the discriminator $d_{\phi}$ achieves logistic loss for each sample
    \[
        -\log\sigma(d_{\phi}(x_i)) \et -\log(1-\sigma(d_{\phi}(x'_i)))\;\leq\; \epsilon.
    \]
    Then there exists a constant $c>0$, depending asymptotically on $\log^2(\epsilon)$ as $\epsilon\to 0$, such that:
    \[
        \lim_{N\rightarrow\infty}\frac{\mathbb{E}\left[\mathcal{L}_{\mathrm{MSE}}^d(\phi)\right]}{N}\;=\; c,\quad\text{with}\quad c \sim (1-TV(P,\wh{P}))^4\log^2(\epsilon),
    \]
    where $\TV(P,\wh{P})$ denotes the total variation distance between $P$ and $\wh{P}$. In other words, when $\TV(P,\wh{P}) >0$, as the discriminator overfits the cross-entropy loss ($\epsilon \to 0$), the mean-squared error scales as $N\log^2(\epsilon)$ and becomes arbitrarily large.
\end{theorem}
\begin{hproof}
    The proof of the theorem is provided in Appendix~\ref{app:sec:proofoverfitting}.
    For clarity, we focus on the one-dimensional case, which explicitly illustrates the link between the similarity of $P$ and $\widehat{P}$ and the discriminator's behavior. However, the argument extends to higher dimensions. Overfitting leads the discriminator to assign highly different values to real and generated samples. As $P$ and $\widehat{P}$ get closer, generated samples increasingly appear near real ones, forcing the discriminator to separate them sharply. This induces high-frequency oscillations in $d_\phi$, where small input changes cause large output variations, resulting in excessive gradients even where the true gradient should be smooth.
\end{hproof}
\subsubsection{Interpretation.}
Theorem~\ref{thm:overfitting} highlights the behavior when the discriminator is \emph{overly confident}, with logits approaching $\pm\infty$. Two distinct scenarios arise:
\begin{itemize}
    \item When $P$ and $\widehat{P}$ differ significantly ($\TV(P ,\widehat{P}) \approx 1$), the discriminator's confidence is well-founded, resulting in stable MSE (Figure~\ref{fig:thm2a}).
    \item However, if the distributions $P$ and $\widehat{P}$ overlap significantly ($\TV(P ,\widehat{P}) \ll 1$), forcing high discriminator confidence ($\epsilon\to 0$) constitutes overfitting. In this scenario, the constant $c$ grows unbounded, causing the gradient MSE to diverge even though the training CE is minimized (Figure~\ref{fig:thm2b}). The situation of overlapping $P$ and $\whP$ is the common framework for refining diffusion models, as the pre-trained model ideally generates a distribution $\whP$ that closely aligns with $P$.
\end{itemize}

\section{Improved Discriminator Guidance}\label{sec:improvement}
In this section, we introduce a discriminator loss designed to explicitly approximate the gradient of the density ratio. We analyze its theoretical properties and practical implications.

\subsection{Proposed loss function}
To overcome the limitations we exposed in Section~\ref{sec:optimality}, we propose to add a term to the loss that explicitly accounts for the gradient of the density ratio. Ideally, if we wanted the gradient of the discriminator to approximate the gradient of the true density ratio, we could optimize the following loss:
\begin{align}\label{eq:scorematchingphi}
    \calL_{\mathrm{SM}}^{d}(\phi) = \int_0^T \lambda(t) \E_{P_0, P_t}\bigg[  \Vert \nabla_{\vx_t} \log p_t(\vx_t)/\whp_t(\vx_t) - \nabla_{\vx_t} \vd_\phi(\vx_t, t) \Vert_2^2 \bigg] \dt.
\end{align}
But, minimizing this loss suffers from the same obstacle as the score matching loss defined in Equation~\eqref{eq:scorematching}: the gradient of the true log-likelihood $\nabla_{\vx_t} \log p_t(\vx_t)$ is unknown. Hopefully, we can use the same argument as \citet{vincent_connection_2011} and instead use the denoising score matching loss Equation~\ref{eq:msephi}. Proposition~\ref{thm:mseloss} shows that minimizing $\calL_{\mathrm{SM}}^{d}(\phi)$ is equivalent to minimize $\calL_{\mathrm{MSE}}^{d}(\phi)$.
\begin{align}\label{eq:msephi}
    \calL_{\mathrm{MSE}}^{d}(\phi) = \int_0^T\lambda(t) \E_{P_0, P_{t\vert\vx_0}}\bigg[ & \Vert \nabla_{\vx_t} \log p_t(\vx_t\vert\vx_0)  - \vs_\theta(\vx_t, t)- \nabla_{\vx_t} \vd_\phi(\vx_t, t)\Vert_2^2 \bigg] \dt.
\end{align}
\begin{proposition}\label{thm:mseloss}
    Assume that $P$ and $\whP$ satisfy the assumptions detailed in Appendix~\ref{app:sec:assumptionssong}.
    Then, the following holds:
    \begin{equation}
        \argmin_\phi \calL_{\mathrm{SM}}^{d}(\phi) = \argmin_\phi \calL_{\mathrm{MSE}}^{d}(\phi).
    \end{equation}
\end{proposition}
\begin{hproof}
    The proof of Proposition~\ref{thm:mseloss} is given in Appendix~\ref{app:sec:proofmseloss}. It follows from the fact that the losses differ only by an additive constant independent of $\phi$.
\end{hproof}
Therefore, training a discriminator with $\calL_{\mathrm{MSE}}^{d}(\phi)$ will correctly make its gradient a reliable estimate of the log-likelihood ratio.

\subsection{Practical Considerations}
In practice, we combine our introduced loss term with the standard cross-entropy loss to facilitate learning, and optimize the loss equation~\ref{eq:finalloss}.
\begin{align}\label{eq:finalloss}
    \calL_{\mathrm{train}}^{d}(\phi) = \calL_{\mathrm{MSE}}^{d}(\phi) + \gamma\calL_{\mathrm{CE}}^{d}(\phi),
\end{align}
where $\gamma$ is a hyperparameter that controls the importance of the cross-entropy loss. We detail the algorithmic implementation of our method in Algorithm \ref{alg:discriminator_training}.

\begin{algorithm}[t!]
    \caption{Training Discriminator with \( \mathcal{L}_{\text{CE}}^d \) and \( \mathcal{L}_{\text{MSE}}^d \)}
    \label{alg:discriminator_training}
    \begin{enumerate}
        \item Input: Pre-trained model $s_\theta$, Set of real data $\mathcal{X}$, Set of generated samples $\hat{\mathcal{X}}$  weighting function $\lambda(t)$, Distribution of timesteps $\mathcal{T}$, batch size $b$
        \item Initialize discriminator parameters \( \phi \)
        \item Repeat until convergence:
              \begin{enumerate}
                  \item Sample a batch $\{\vx_i\}_{i=1}^b$ from the training data $\mathcal{X}$ and $\{t\}_{i=1}^b \sim \mathcal{T}$
                  \item Perturb the samples $\{\vx_i\}_{i=1}^b$ to timesteps $t_i$ to $\{\vx_i^t\}_{i=1}^b$
                  \item \textbf{Cross-Entropy Optimization (Baseline Loss):}
                        \begin{itemize}
                            \item Sample a batch $\{\hat{\vx}_i\}_{i=1}^{b}$ from $\hat{\mathcal{X}}$ and perturb to timesteps $t_i$: $\{\hat{\vx}_i^{t}\}_{i=1}^{b}$
                            \item Compute the CE loss:
                                  $$\widehat{\mathcal{L}_{\text{CE}}^d(\phi)} = \frac{1}{b}\sum_{i=1}^{n}\lambda(t_i)\left[-\log\left(\sigma\left(d_{\phi}(\vx_i^{t}, t)    \right)    \right) - \log\left(1-\sigma\left(d_{\phi}(\hat{\vx}_i^{t}, t)    \right)    \right) \right].$$
                        \end{itemize}
                  \item \textbf{Gradient Matching Optimization (Proposed Loss):}
                        \begin{itemize}

                            \item Compute target and pre-trained model's scores on the perturbed training batch $\{\vx_i^t\}_{i=1}^b$: \( \nabla_{\vx_i^t} \log p_t(\vx_i^{t}\vert\vx_i) \) and $s_{\theta}(x_i^{t},t)$
                            \item Compute discriminator gradient: \( \nabla_{\vx_i^t} d_\phi(\vx_i^{t}, t) \) (e.g autodifferentiation)
                            \item Compute gradient-matching loss:
                                  \[
                                      \widehat{\mathcal{L}_{\text{MSE}}^d(\phi)} =  \frac{1}{b} \sum_{i=1}^b \lambda(t_i)\| \nabla_{\vx_i^t} \log p_t(\vx_i^t \vert \vx_i) - s_{\theta}(\vx_i^t,t) - \nabla_{\vx_i^t} d_\phi(\vx_i^t, t) \|^2
                                  \]

                        \end{itemize}
                  \item Update \( \phi \) via gradient descent on \(\widehat{\mathcal{L}_{\text{train}}^d(\phi)} = \widehat{\mathcal{L}_{\text{CE}}^d(\phi)} + \gamma \widehat{\mathcal{L}_{\text{MSE}}^d(\phi)} \)
              \end{enumerate}
    \end{enumerate}
\end{algorithm}

\subsubsection{Optimizing \( \mathcal{L}_{\text{CE}}^d(\phi)\) and \( \gamma \mathcal{L}_{\text{MSE}}^d(\phi) \)}
Optimizing both terms yields to different considerations:
\begin{itemize}
    \item \textbf{Optimizing \( \mathcal{L}_{\text{CE}}^d(\phi)\).} In \citet{kim_refining_2023}, the discriminator is trained on real and generated samples, drawn from the target distribution \(P\) and the refined distribution \(\widehat{P}\). Since generated samples are typically precomputed and stored, training with \( \mathcal{L}_{\text{CE}}^d(\phi)\) exposes the discriminator to a fixed dataset, increasing the risk of overfitting. However, cross-entropy loss is easier to optimize than MSE loss, as it does not require computing the target distribution gradient. This results in faster training and lower memory usage.
    \item \textbf{Optimizing \(\mathcal{L}_{\text{MSE}}^d(\phi)\).} Our proposed loss introduces dynamic perturbations to training samples, exposing the model to greater variability. While it correctly estimates the target gradient, this comes with computational overhead: each sample requires computing both \(\nabla_{\vx_t} \log p_t(\vx_t\vert\vx_0)\) and \(\vs_\theta(\vx_t, t)\). Additionally, backpropagation is more complex, as gradients must be propagated through \(\nabla_{\vx_t} d_\phi(\vx_t, t)\) rather than directly through \(d_\phi\). These factors lead to increased computational time and memory requirements.
\end{itemize}

\section{Experiments}\label{sec:experiments}

In this section we compare our proposed method to the baseline from \citet{kim_refining_2023}. We demonstrate the benefits of our approach on both synthetic and real-world datasets. We will (1) observe the effect of overfitting on the quality of the refinement, (2) show the behavior of the training loss on the discriminator guidance, and (3) compare effectiveness of the proposed method for different training/generation settings and finally (4) compare the quality of the samples generated by the EDM, the EMD+DG and our method.

\begin{figure*}[b!]
    \centering
    \includegraphics[width=\textwidth]{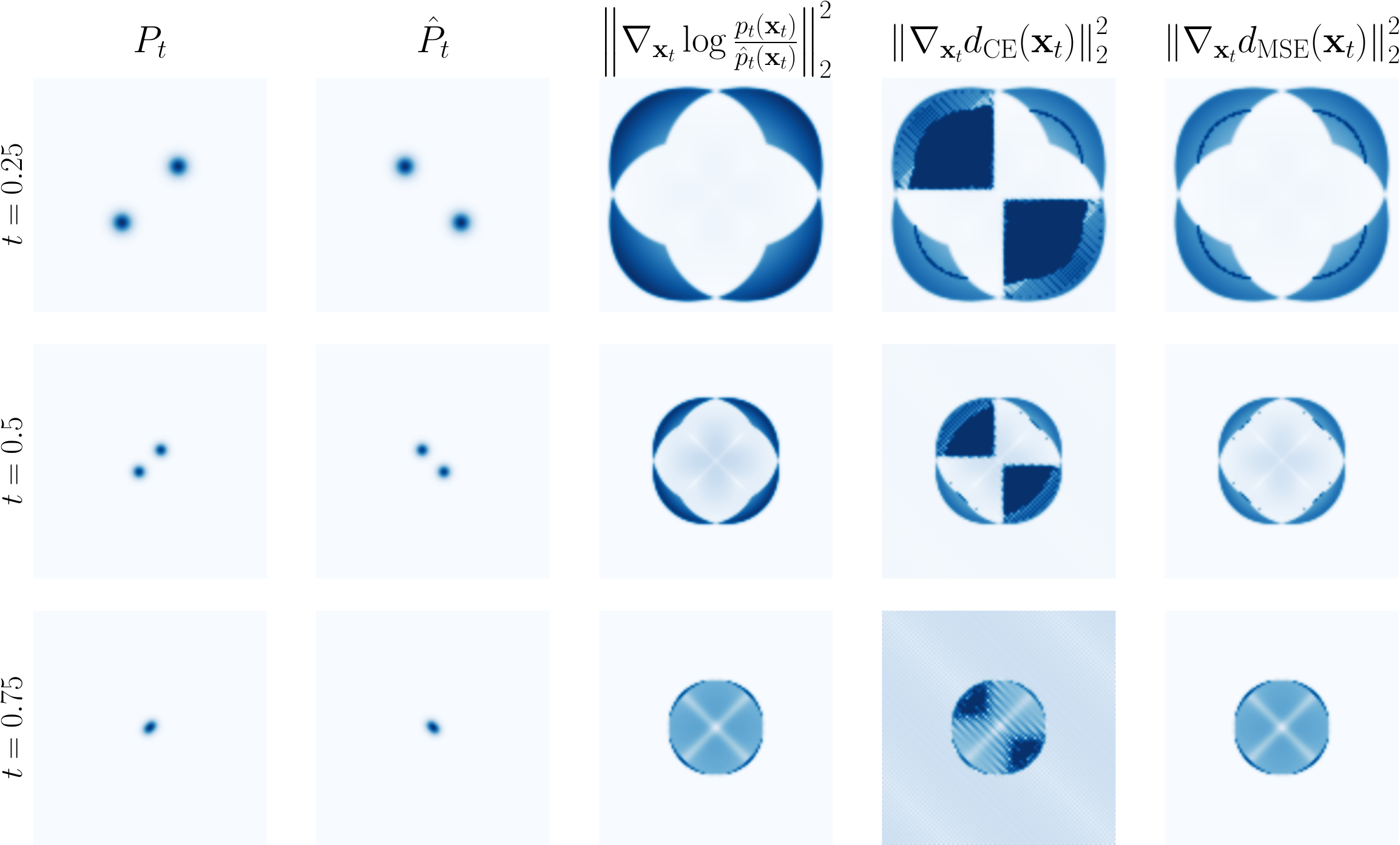}
    \caption{Visualizing the estimation of $\nabla_{\vx_{t}} \log p_{t}(\vx_{t})/ \log \hat{p}_{t}(\vx_{t})$ for the discriminator trained with low CE or low MSE loss. We plot the norms for better readability.}
    \label{fig:2d_gradient}
\end{figure*}
\subsection{Visualizing the Discriminator Guidance in low-dimension}
We consider two distinct Gaussian mixtures in $\mathbb{R}^{2}$, representing $P$ and $\widehat{P}$. Using a subVP-SDE \citep{song_score-based_2021}, we derive the closed-form expression of the score ratio $\nabla_{\vx} \log p(\vx)/ \widehat{p}(\vx)$ and employ it to refine samples from $\widehat{P}$ towards $P$. We evaluate the performance of a discriminator trained with cross-entropy loss $\mathcal{L}^d_{\mathrm{CE}}$ and mean squared error loss $\mathcal{L}^d_{\mathrm{MSE}}$.

\subsubsection{Training with $\mathcal{L}^d_{\mathrm{MSE}}$ gives a better gradient approximation.} We plot the resulting gradient norms in Figure~\ref{fig:2d_gradient} (full vector fields are depicted in Figure~\ref{fig:2D-quivers} in Appendix~\ref{app:sec:additionalplots}). On this synthetic examples, we see that the discriminator trained with $\mathcal{L}^d_{\mathrm{MSE}}$ achieves a much more precise gradient estimation, resulting in improved samples refinement than the discriminator trained with $\mathcal{L}^d_{\mathrm{CE}}$. This is confirmed in Figure~\ref{fig:2Dplots}, where we compare the estimated density of the refined distribution for both methods, demonstrating that the MSE loss yields superior refinement quality.


\subsection{Testing our approach the methods on real-world dataset}

For high-dimensional image datasets, we implement our approach using the unconditional pre-trained EDM model \citep{karras_elucidating_2022} on CIFAR-10, FFHQ in resolution $64\times 64$, and AFHQv2 in resolution $64\times64$. We apply the generation algorithm introduced by \citet{kim_refining_2023} to the EDM framework. To generate samples from the pre-trained model with a discriminator, they introduce a weight $w$ to balance the contributions of the pre-trained model and the discriminator:
\begin{equation}
    \widetilde{\vs}_\theta(\vx) = \vs_{\theta}(\vx) + w \nabla_{\vx} d_{\phi}(\vx).
\end{equation}
As a baseline, we use a discriminator trained with cross-entropy loss $\mathcal{L}^d_{\mathrm{CE}}$, denoted as (EDM+DG), and compare it with our method, which employs $\mathcal{L}^d_{\mathrm{train}}$ introduced in Equation~\eqref{eq:finalloss}. The generation processes are evaluated using the Fréchet Inception Distance (FID) \citep{heusel_gans_2017}, as well as Precision and Recall with $k=3$ \citep{kynkaanniemi_improved_2019}. For FID computation, we use 50k samples for the CIFAR-10 and FFHQ datasets and 15k samples for the AFHQv2 dataset. For Precision and Recall, we set $k=3$ and use 10k samples for each dataset.

Finally, the discriminators are parametrized following \citet{kim_refining_2023}: we adopt the ADM architecture, freezing the upper layers and training only the final layers. This results in training only 2.88M parameters, compared to the total of 50.6M, 68.3M, and 68.3M parameters for CIFAR-10, FFHQ, and AFHQv2, respectively. For comparison, the pre-trained EDM model consists of 55.7M, 61.8M, and 61.8M parameters for these datasets. Generation is conducted on 2×GPU-H100, while evaluation is performed on 2×GPU-V100. A complete list of hyperparameters is available at \url{https://github.com/AlexVerine/BoostDM}.

\begin{figure*}[b!]
    \centering
    \includegraphics[width=\textwidth]{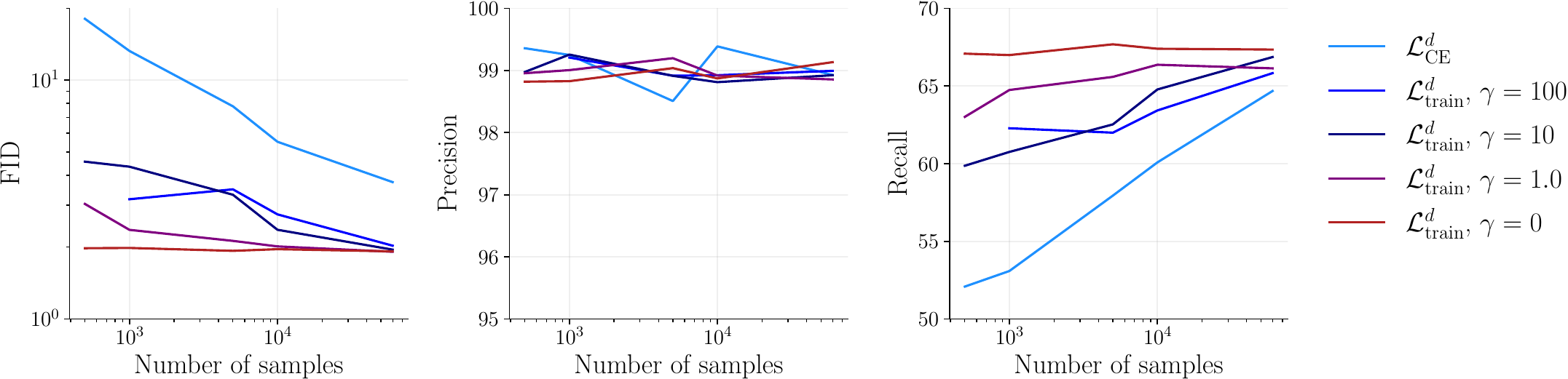}
    \caption{Discriminator guidance using a different number of samples for the training set. Generation is performed on CIFAR-10 with $w=1$ for all methods. FID ($\downarrow$), Precision ($\uparrow$), and Recall ($\uparrow$) are reported.}
    \label{fig:size}
\end{figure*}

\subsubsection{Observing Overfitting in the Discriminator.}
We analyze how the training loss of the discriminator evolves on CIFAR-10 by varying the number of training samples from 500 to 50,000. Figure~\ref{fig:size} presents the key metrics. While Precision remains relatively stable regardless of dataset size, both Recall and the FID score change notably: Recall decreases as the number of training samples increases, while the FID score rises. This indicates that the discriminator overfits the training set, ultimately degrading the refinement quality.
Additionally, our proposed method shows greater robustness to dataset size. However, when the balance between MSE and CE losses increases (i.e., for higher $\gamma$ values), sensitivity to the number of training samples also increases.
\subsubsection{Observing the effect of the regularization parameter $\gamma$.} We compare how the estimated score behaves with the two different training losses. To assess how the discriminator captures the gradient of the log density ratio, we plot in Figure~\ref{fig:losses} the evolution during of $\mathcal{L}^{d}_{\mathrm{MSE}}$ introduced in Equation~\eqref{eq:msephi}. We observe that the lower $\gamma$ is, the better the discriminator is at estimating the gradient of the log density ratio. This is consistent with the results of the previous section.
\begin{figure*}[t!]
    \centering
    \includegraphics[width=\textwidth]{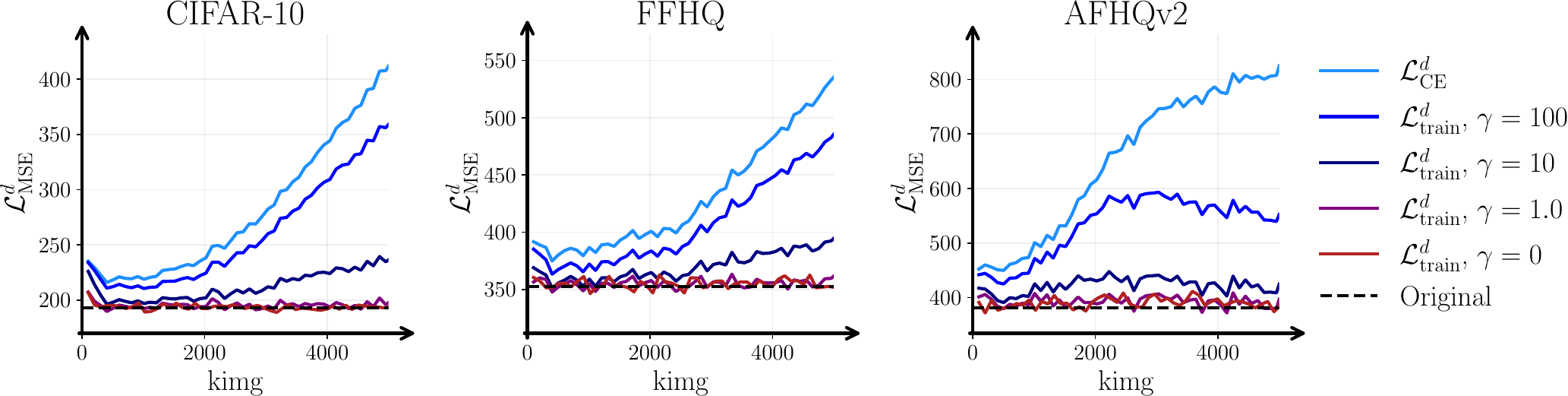}
    \caption{Comparison of the proposed loss $\mathcal{L}^{d}_{\mathrm{train}}$ and the standard loss $\mathcal{L}^{d}_{\mathrm{CE}}$ for the discriminator guidance. We plot the evolution of the loss during training.}
    \label{fig:losses}
\end{figure*}

\subsubsection{Comparing the effect of the weight $w$.} Depending on the training loss, the gradient of the discriminator can have different ranges and for the refinement to be effective, the weight $w$ must be adjusted. We evaluate the effect of $w$ on the FID score for the EDM+DG and EDM+Ours methods on CIFAR-10, FFHQ, and AFHQv2. The results are shown in Figure~\ref{fig:fid_weights}. We observe that the optimal $w$ is different for each dataset but that the proposed method typically leads to lower effect of the discriminator and therefore a larger $w$ is needed. For each method and dataset, we report the FID, Precision, and Recall scores for the parameters $w$ and $\gamma$ that yield the best FID score in Table~\ref{tab:results}.

\begin{figure*}[b!]
    \includegraphics[width=\textwidth]{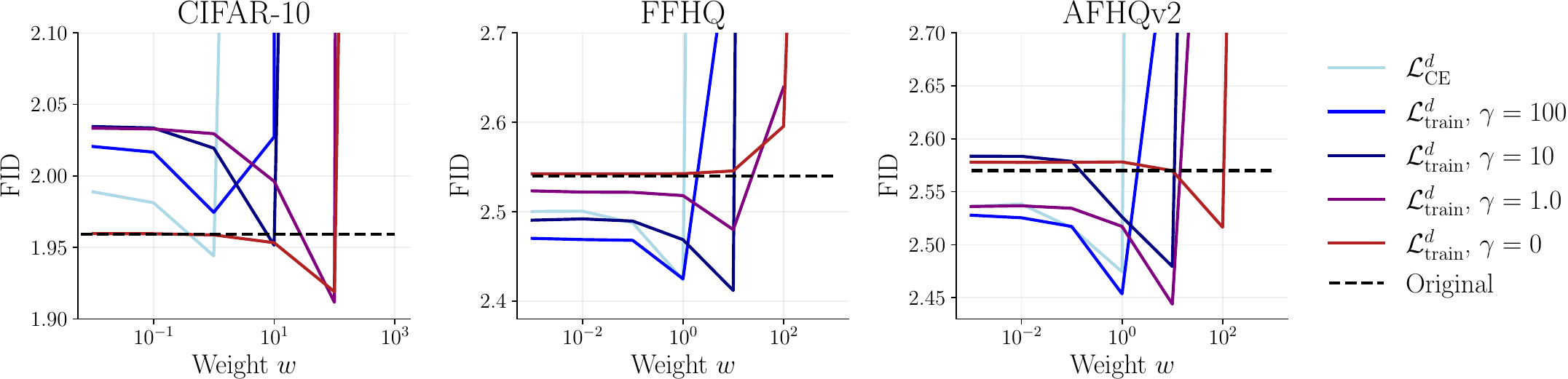}
    \caption{FID score for different values of $w$ on CIFAR-10, FFHQ, and AFHQv2.}
    \label{fig:fid_weights}
\end{figure*}

\begin{table*}[t!]
    \centering
    \caption{Comparison of the proposed method with optimal $w$ on CIFAR-10, FFHQ and AFHQv2. We report the FID ($\downarrow$), Precision ($\uparrow$), and Recall ($\uparrow$) scores and the number of GPUs used for training and generation, the time for training step in seconds per 1000 images, and the memory consumption in Gi.}
    \label{tab:results}
    \begin{tabular*}{\textwidth}{l@{\extracolsep{\fill}}lccc|ccc}
        \toprule
        Dataset &  Method & FID & P & R & GPUs & Time & Mem. (Gi) \\
        \midrule
        \multirow{3}{*}{CIFAR-10} &   EDM & $1.96$ & $99.10$ & $67.48$ &-&-&- \\
        & EDM+DG & $1.94$ & $98.91$ & $67.44$&4xV100&1.18&2.24 \\
        & EDM+Ours & $1.91$ & $98.86$ & $66.14$&4xV100&$4.00$&$9.18$\\
        \midrule
        \multirow{3}{*}{FFHQ} & EDM & $2.54$ & $99.69$ & $69.69$ &-&-&-\\
        & EDM+DG & $2.42$ & $99.60$ & $69.16$&4xH100&1.05&7.03 \\
        & EDM+Ours & $2.41$ & $99.56$ & $69.09$&4xH100&4.74&20.26 \\
        \midrule
        \multirow{3}{*}{AFHQv2} &   EDM & $2.57$ & $99.99$ & $75.64$ &-&-&-\\
        & EDM+DG & $2.47$ & $99.83$ & $74.41$&4xH100&1.05&7.03 \\
        & EDM+Ours & $2.44$ & $99.96$ & $74.58$&4xH100&4.74&20.26 \\
        \bottomrule
    \end{tabular*}
\end{table*}

\begin{figure*}[b!]
    \subfloat[EDM]{\includegraphics[width=0.32\textwidth]{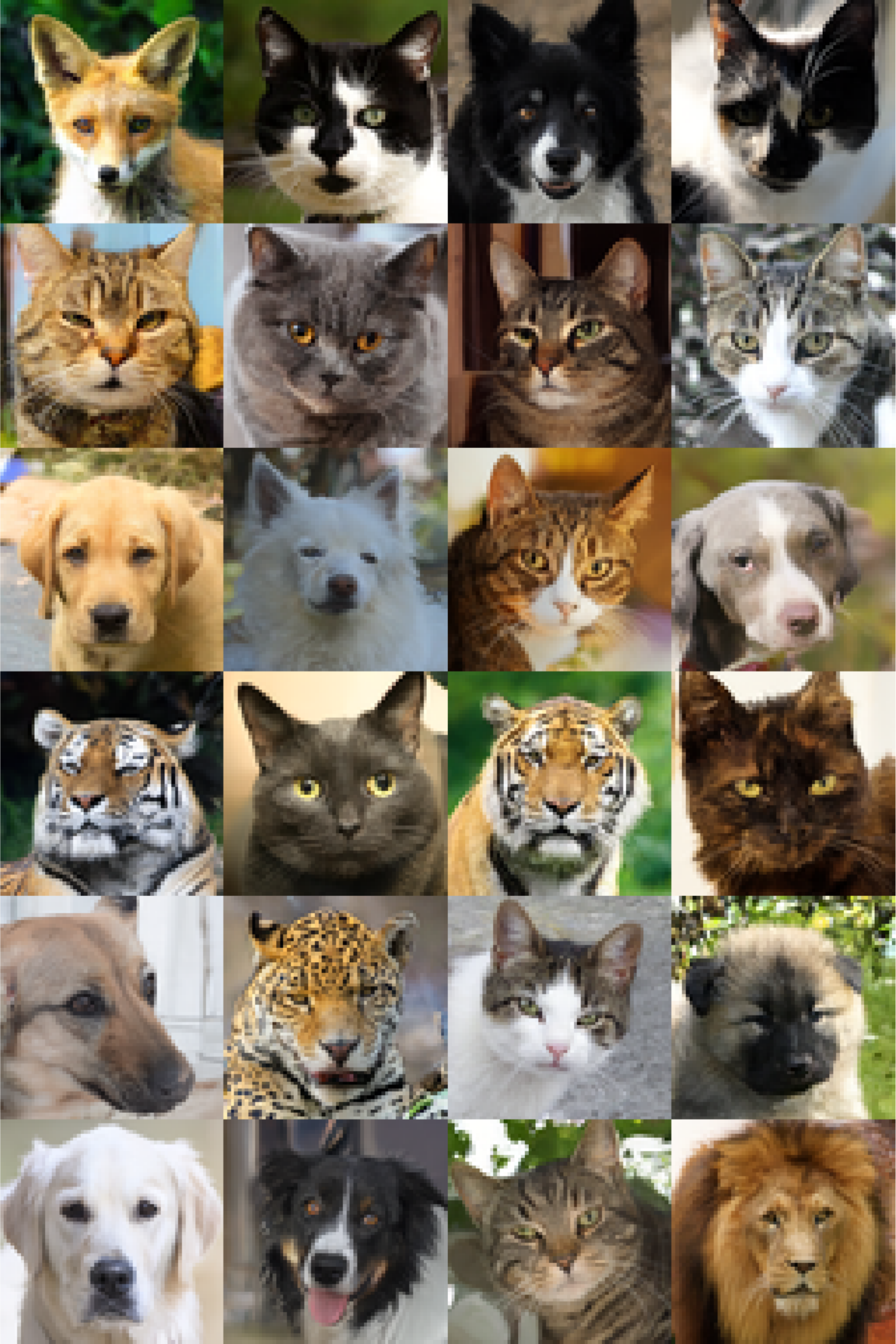}}\hfill
    \subfloat[EDM+DG]{\includegraphics[width=0.32\textwidth]{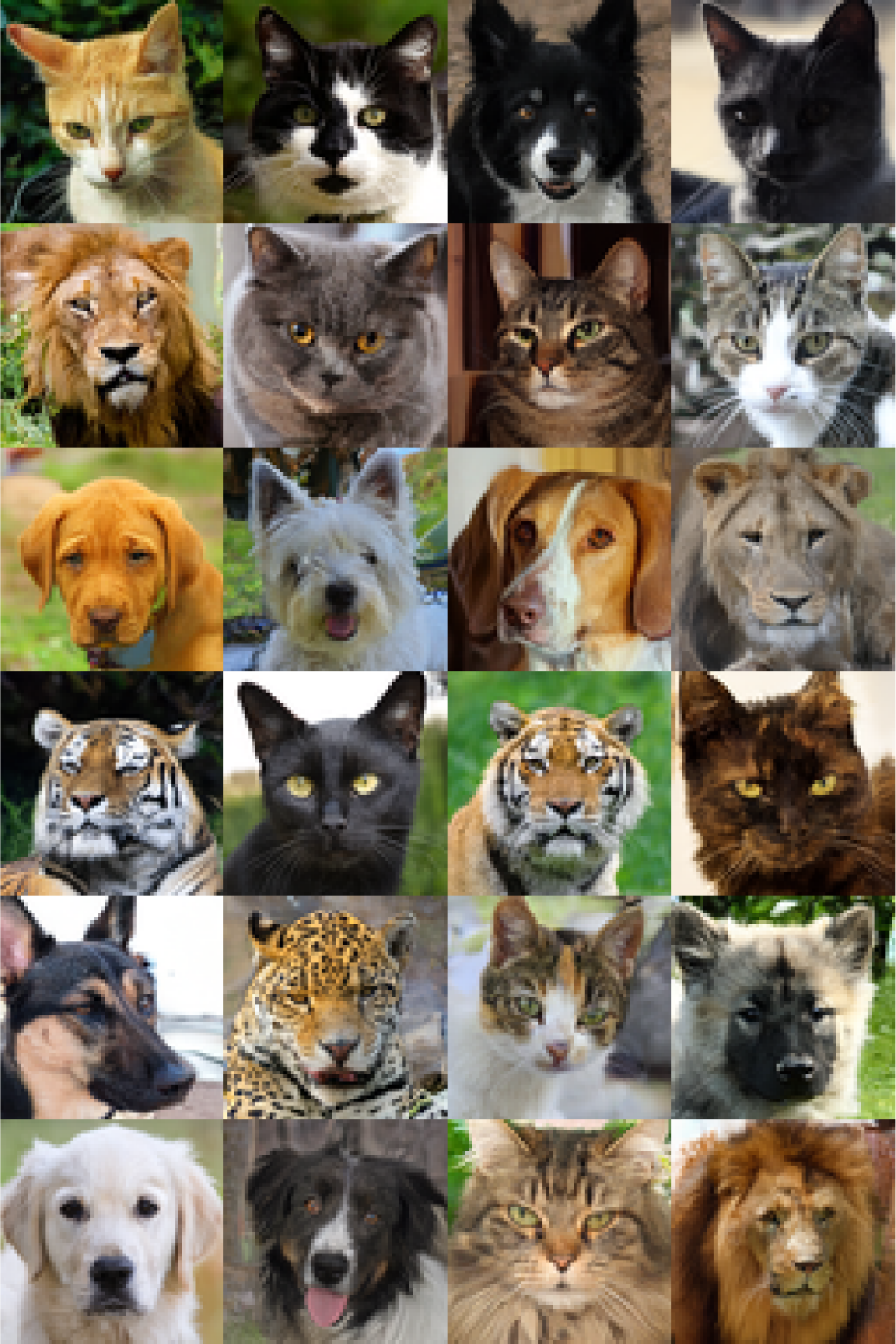}}\hfill
    \subfloat[EDM+Ours]{\includegraphics[width=0.32\textwidth]{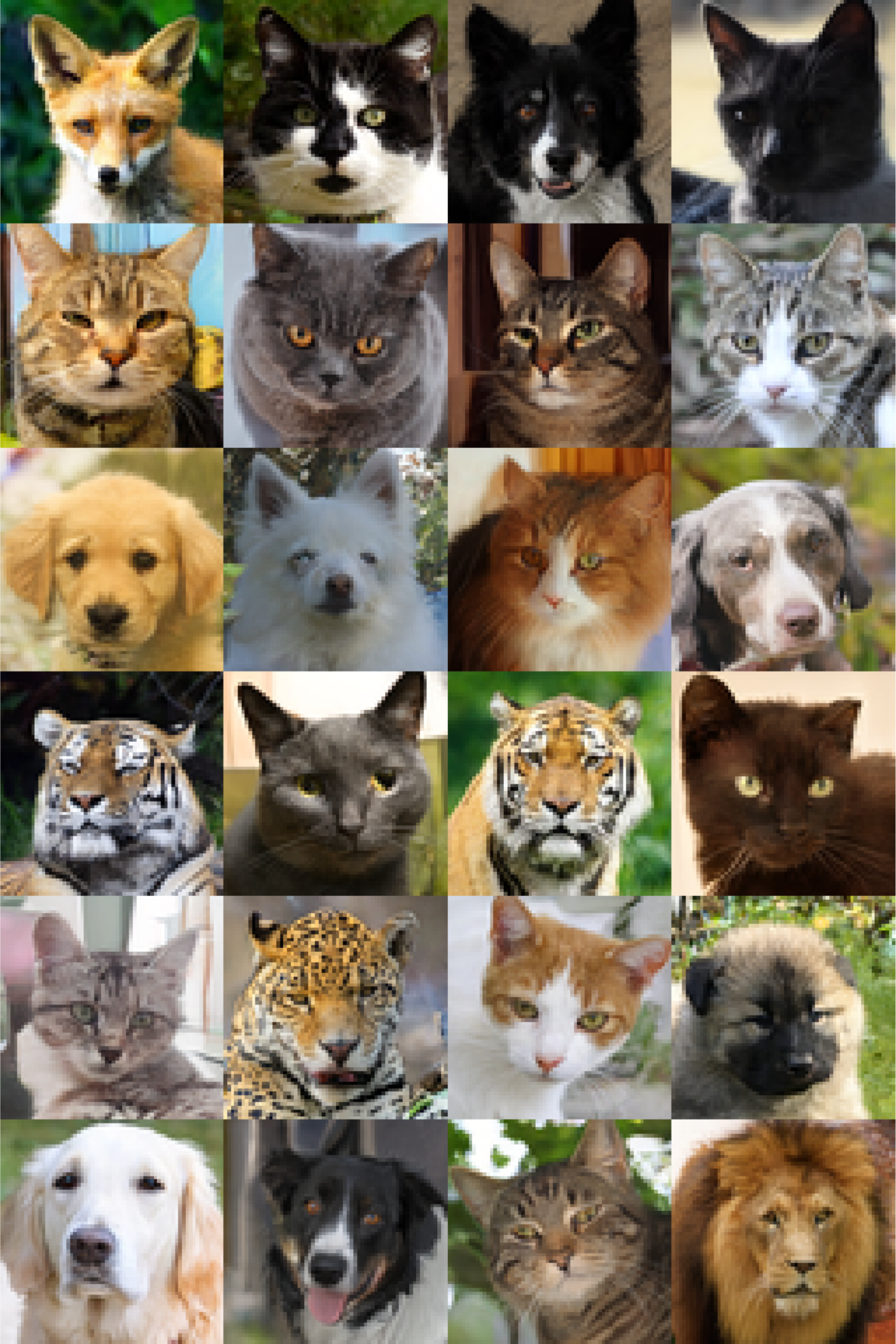}}
    \caption{Samples generated on the AFHQv2 dataset.}
    \label{fig:afhq}
\end{figure*}

\subsubsection{Comparing EDM, EDM+DG and our method.} In Table~\ref{tab:results}, we compare the FID, Precision, and Recall scores for the EDM, EDM+DG, and EDM+Ours methods on CIFAR-10, FFHQ, and AFHQv2. We observe that our method consistently outperforms the EDM+DG method in terms of FID score. The Precision is not significantly affected by the method used, while the Recall can be slightly lower for our method. We can observe on Figure~\ref{fig:afhq} that the samples generated by our method are closer to the samples generated by the EDM model than the EDM+DG method. On this uncurated set of samples from the AFHQv2 dataset, we can see that the samples generated by the EDM+DG method often changes the class of the sample while our method does not. We plot more samples in Appendix~\ref{app:sec:additionalplots}.

\section{Conclusion and Discussion}\label{sec:conclusion}

In this paper, we have demonstrated that discriminator guidance can be significantly improved by refining the loss function used for training the discriminator. Our theoretical analysis reveals that minimizing cross-entropy can lead to unreliable refinement, particularly in the overfitting regime, where the discriminator learns to separate real and generated samples too aggressively. To mitigate this issue, we introduced an alternative training approach that minimizes the reconstruction error, resulting in more accurate gradient estimation and improved sample quality in EDM diffusion models across diverse datasets.

Our findings highlight a fundamental challenge: while a discriminator can estimate the density ratio, accurately capturing its gradient remains difficult in practice. This issue is exacerbated by overfitting, where the discriminator's learned function develops high-frequency oscillations that distort the refinement process. Although discriminator guidance is theoretically optimal with a perfect discriminator, real-world constraints—such as finite data, model capacity, and training instability—limit its practical effectiveness.


\begin{credits}
    \subsubsection{\ackname}
    This work was granted access to the HPC resources of IDRIS under the allocations 2025-A0181016159, 2025-AD011014053R2 made by GENCI.

\end{credits}
%

%
%
%
\bibliographystyle{splncs04}
\bibliography{references,references_2}

\appendix
\appendix
\section{Mathematical Supplementary}
\subsection{Assumptions}\label{app:sec:assumptionssong}
In the paper and in the following proofs, we make the following assumptions :
\begin{enumerate}
    \item $p(\vx) \in \cal C^2(\reals^d)$ and $\E_P\left[\Vert \vx \Vert_2^2\right] < \infty$.
    \item $q(\vx) \in \cal C^2(\reals^d)$ and $\E_Q\left[\Vert \vx \Vert_2^2\right] < \infty$.
    \item $\forall t \in [0, T]$, $\vf(\vx, t) \in \cal C^1(\Xset)$ and $\exists C>0$ such that $\forall \vx \in \reals^d,t\in[0, T],\Vert \vf(\vx, t)\Vert_2 \leq C( 1+\Vert\vx\Vert_2)$.
    \item $\exists C >0 $ such that $\forall \vx, \vy \in \reals^d, \ \Vert \vf(\vx, t) - \vf(\vy, t)\Vert_2 \leq C\Vert \vx - \vy\Vert_2$.
    \item $g\in \cal C^1([0, T])$ and $\forall t \in [0, T],\ \vert g(t)\vert > 0$.
    \item For any open bounded set $\cal O \subset \reals^d$, $\int_0^T \int_{\cal O} \Vert p_t(\vx_t) \Vert_2^2 + d\Vert \nabla_{\vx_t}p_t(\vx_t) \Vert_2^2\d \vx_t \d t < \infty$.
    \item $\exists C >0$ such that $\forall \vx\in \reals^d, t\in[0,T]:\ \Vert\nabla_{\vx_t} \log p_t(\vx_t)\Vert_2 \leq C(1+\Vert \vx_t\Vert_2)$.
    \item $\exists C >0$ such that $\forall \vx,\vy\in \reals^d,t\in[0,T]:\ \Vert\nabla_{\vx_t} \log p_t(\vx_t) - \nabla_{\vx_t} \log p_t(\vy_t)\Vert_2 \leq C\Vert \vx_t - \vy_t\Vert_2$.
    \item $\exists C>0$ such that $\forall \vx\in \reals^d, t\in[0,T]:\ \Vert\nabla_{\vx_t} \vs_\theta(\vx_t, t)\Vert_2 \leq C(1+\Vert \vx_t\Vert_2)$.
    \item $\exists C>0$ such that $\forall \vx,\vy\in \reals^d,t\in[0,T]:\ \Vert\nabla_{\vx_t} \vs_\theta(\vx_t, t) - \nabla_{\vx_t} \vs_\theta(\vy_t, t)\Vert_2 \leq C\Vert \vx_t - \vy_t\Vert_2$.
    \item Novikov's condition: $\E_P\left[\exp\left(\frac{1}{2}\int_0^T\Vert \nabla_{\vx_t} \log p_t(\vx_t)-\vs_\theta(\vx_t, t)\Vert_2^2\d t\right)\right] < \infty$.
    \item  $\forall t\in[0, T], \exists k>0: \ p_t(\vx)=O(e^{-\Vert \vx \Vert_2^k})$ as $\Vert \vx \Vert_2 \to \infty$.
\end{enumerate}

\subsection{Proof of Theorem~\ref{thm:suboptimal}}\label{app:sec:proofsuboptimal}

\begin{theorem}
    Let $\left\{\vx(t)\right\}_{t\in[0, T]}$ be a diffusion process defined by Equation~\eqref{diffproc}. Assume that $P$ and $\whP$ satisfy the assumptions detailed in Appendix~\ref{app:sec:assumptionssong}. Then, for every $\varepsilon>0$ and for every $\delta>0$, there exists a discriminator $\vd:\reals^d\times\reals\to\reals$ trained to minimize the cross-entropy such that:
    \begin{equation}
        \calL_{\mathrm{CE}}^{d}(\phi)\leq \varepsilon \et \KL(P\Vert \widetilde{P}) \geq \delta,
    \end{equation}
    where $\widetilde{P}$ is the distribution induced by discriminator guidance with $\vd$.
\end{theorem}

\subsubsection{Finding a problematic case}
The aim of discriminator guidance is to minimize the Kullback-Leibler divergence between the target distribution $P$, and the refined distribution $\tilde{P}$.
Following the assumptions detailed in Appendix \ref{app:sec:assumptionssong}, this quantity can be written as :
\begin{align}
    \KL(P\Vert \wtP) & = \KL(P_T\Vert Q) + \\&\int_0^Tg(t)^2\E_{P_t}\left[\Vert \nabla_{\vx_t} \log p_t(\vx_t) - \vs_\theta(\vx_t, t) - \nabla_{\vx_t}\vd_\phi(\vx_t, t)\Vert_2^2 \right]\dt.
\end{align}
Thus, discriminator guidance minimizes the latter term of the equality, that we denote as :

\begin{align}\label{main-goal}
    E_{\phi} =\int_0^Tg(t)^2\E_{P_t}\left[\Vert \nabla_{\vx_t} \log p_t(\vx_t) - \vs_\theta(\vx_t, t) - \nabla_{\vx_t}\vd_\phi(\vx_t, t)\Vert_2^2 \right]\dt.
\end{align}
where $d_{\phi}(\vx_{t})$ is the estimated density ratio by training the discriminator $d_{\phi}$.

We consider the case of estimating $d_{\phi}$ by minimizing the cross-entropy up to $\varepsilon$. We would like to find a pathological case where this would not minimize the mean square error in Equation \ref{main-goal}.
The estimation error of the discriminator can be derived from the duality difference in estimating the g-Bregman divergence, with :
\begin{align}
    g(\vx) = u\log(u) + (u+1)\log(u+1)
\end{align}
We denote by $\mathcal{D}_{g}(P \Vert P') = \inf_{f \in \mathcal{M}} \mathbb{E}_{(\vx) \sim p_{t}}\left[\log(\sigma(f(\vx_{t})))\right] - \mathbb{E}_{\vx \sim \hat{p}_{t}}\left[-\log(1 - \sigma(f(\vx_{t})))\right]$, where $\mathcal{M}$ denotes the set of all measurable functions. This quantity estimates the optimal cross-entropy that any measurable density ratio estimator $f$ can achieve.
The estimated $d_{\phi}$ has minimized the following quantity
\begin{align*}
    \mathcal{D}_{g}^{\phi} = \inf_{\omega \in R^{N}}\mathbb{E}_{\vx \sim p_{t}}\left[\log(\sigma(T_{\omega}(\vx_{t})))\right] - \mathbb{E}_{\vx \sim \hat{p}_{t}}\left[-\log(1 - \sigma(T_{\omega}(\vx_{t})))\right]
\end{align*}
We use the results of \citep{uehara_generative_2016,sugiyama_density_2010,verine_precision-recall_2023} to compute the estimation error of the cross-entropy:

\begin{theorem}
    For any discriminator $T_{\phi} : \mathcal{X} \xrightarrow{} R$ and density ratio estimation $d_{\phi} = \nabla_{\vx_t} g^{*}(T_{\phi}(\vx))$, we have :
    \begin{align}\label{thm1}
        D_{g}(P||\hat{P}) - D_{g}^{\phi}(P||\hat{P}) = \mathbb{E}_{\hat{P}}
        \left[
            \text{Breg}_{g}\left(d_{\phi}(\vx),\frac{p(\vx)}{\hat{p}(\vx)}\right)
            \right]
    \end{align}
    where $g^*$ denotes the Fenchel conjugate of $g$.
\end{theorem}

Suppose a discriminator was trained on minimizing the cross-entropy $D_{\text{g}}^{\phi}$ such that it is $\varepsilon$-close to the optimal cross-entropy $D_{\text{g}}$. Thus, we can write : \begin{equation} \label{eq:bregmaneps}
    D_{g}(P \Vert \hat{P}) - D_{g}^{\phi}(P \Vert \hat{P}) \leq \varepsilon
\end{equation}
We consider the case when, for all $(\vx) \sim \hat{P}$, $\text{Breg}_{g}(d_{\phi}(\vx),\frac{p(\vx)}{\hat{p}(\vx)}) \leq \varepsilon$
Suppose moreover that the estimated density ratio $d_{\phi}(\vx) = \frac{p(\vx)}{\hat{p(\vx)}} + h(\vx)$. We also note $\frac{p(\vx)}{\hat{p(\vx)}} = r_{\opt}(\vx)$. Thus, we can write, for $\vx \in \mathbb{R}$:
\[
    \text{Breg}_{g}\left(d_{\phi}(\vx),\frac{p(\vx)}{\hat{p}(\vx)}\right) = g(d_{\phi}(\vx)) - g(r_{\opt}(\vx)) - \nabla_{\vx_t} g(r_{\opt}(\vx))(d_{\phi}(\vx) - r_{\opt}(\vx)))
\]
With :
\begin{align}
    g(d_{\phi}(\vx)) & =\begin{aligned}[t] (r_{\opt}(\vx) & + h(\vx))\log(r_{\opt}(\vx) + h(\vx))                         \\
                                  & - (r_{\opt}(\vx) + h(\vx) +1 )\log(r_{\opt}(\vx) + h(\vx) +1)\end{aligned}
    \\
                     & =\begin{aligned}[t]
                            (r_{\opt}(\vx) & + h(\vx))\log (r_{\opt}(\vx))                                              \\
                                           & + (r_{\opt}(\vx) + h(\vx))\log\left(1+ \frac{h(\vx)}{r_{\opt}(\vx)}\right) \\
                                           & - (r_{\opt}(\vx) + h(\vx) +1)\log(r_{\opt}(\vx) + h(\vx) +1)\end{aligned}
    \\
                     & \leq
    \begin{aligned}              & (r_{\opt}(\vx) + h(\vx))\log(r_{\opt}(\vx)) + (r_{\opt}(\vx)+h)\frac{h(\vx)}{r_{\opt}(\vx)} \\
                             & - (r_{\opt}(\vx) + h(\vx) +1)\log(r_{\opt}(\vx) + h(\vx) +1)\end{aligned}
    \label{eq:dl}
\end{align}
Moreover,
\begin{align}
    g(r_{\opt}(\vx)) &= r_{\opt}(\vx)\log(r_{\opt}(\vx)) + (r_{\opt}(\vx) +1)\log(r_{\opt}(\vx)+1)\label{eq:sec-part} \end{align}
and the third term is given by
\begin{align}
    \nabla_{\vx_t} g(r_{\opt}(\vx))(d_{\phi}(\vx) - r_{\opt}(\vx))) & = (\log(r_{\opt}(\vx)) - \log(r_{\opt}(\vx)+1))h(\vx)\label{eq:grad-part}
\end{align}
Inequality \ref{eq:dl} results from the Taylor expansion of $\log(1+(\vx))$ when $(\vx)$ is close to zero, as $\log(1+(\vx)) \leq (\vx)$.
By summing the expressions \ref{eq:dl},\ref{eq:sec-part},\ref{eq:grad-part}, we obtain an upper bound on the $g$-Bregman divergence between the true and estimated density ratio :
\begin{align}
    \text{Breg}_{g}\left(d_{\phi}(\vx),r_{\opt}(\vx)\right) & \leq
    \begin{aligned}[t]
         & \left(1+\frac{h^{2}(\vx)}{r_{\opt}(\vx)} \right) + \left(r_{\opt}(\vx) +1 \right) \log\left( r_{\opt}+1\right) \\
         & + h(\vx)\log\left(r_{\opt}(\vx) +1 \right)                                                                     \\
         & - \left(r_{\opt}(\vx) + h(\vx) + 1\right)\log\left( r_{\opt}(\vx) + h(\vx) + 1\right)\end{aligned}
    \\
                                                            & =\begin{aligned}[t]
                                                                    & \left(h(\vx) + r_{\opt}(\vx)\right)\frac{h(\vx)}{r_{\opt}(\vx)}                            \\
                                                                    & +\left[r_{\opt}(\vx) + h(\vx) + 1\right]                                                   \\
                                                                    & \times\left[\log(r_{\opt}(\vx)+1) - \log(\left(r_{\opt}(\vx) + h(\vx) + 1\right))  \right]
                                                               \end{aligned}
    \\
                                                            & =\begin{aligned}[t]
                                                                   (h(\vx) & + r_{\opt}(\vx))\frac{h(\vx)}{r_{\opt}(\vx)}                                                            \\
                                                                           & +\left(r_{\opt}(\vx) + h(\vx) + 1\right)\log\left( 1 - \frac{h(\vx)}{r_{\opt}(\vx) + h(\vx) + 1}\right)
                                                               \end{aligned}
    \\
                                                            & \leq
    \begin{aligned}[t]
         & (h(\vx) + r_{\opt}(\vx))\frac{h(\vx)}{r_{\opt}(\vx)}                 \\
         & - (r_{\opt}(\vx) + h(\vx) +1)\frac{h(\vx)}{r_{\opt}(\vx)+h(\vx) + 1}
    \end{aligned}
    \\
                                                            & \leq \frac{h(\vx)^{2}}{r_{\opt}(\vx)}
\end{align}
One particular case of satisfaction of inequality \ref{eq:bregmaneps} is when all the elements i the expectation $\mathbb{E}_{\hat{P}}$ are below $\varepsilon$, that is when : \begin{align}
    \frac{h^{2}(\vx)}{r_{\opt}(\vx)} & \leq \varepsilon                      \\
    \Rightarrow h(\vx)               & \leq \sqrt{\varepsilon r_{\opt}(\vx)}
\end{align}
This bound is notably satisfied for $h(\vx) = \sin(\omega \vx)\sqrt{\varepsilon r_{opt}(\vx)} $, $\forall \omega \in \mathbb{R}$
\subsubsection{Computing the gain}
We will now show that the value of the $E_\phi$ (c.f Equation \ref{main-goal}) can go to infinity for an estimated density ratio that has an $\varepsilon$-optimal cross-entropy. For this, we write, for $r(\vx) = r_{opt}(\vx) + h(\vx)$, with $h(\vx) =\sin(\omega (\vx))\sqrt{\varepsilon r_{opt}(\vx)} $ :
This bound is notably satisfied for $h(\vx) = \sin(\omega (\vx))\sqrt{\varepsilon r_{\opt}(\vx)} $, $\forall \omega \in \mathbb{R}$
\subsection{Computing the gain}
We will now show that the value of the gain (c.f Equation \ref{main-goal}) can go to infinity for an estimated density ratio that has an $\varepsilon$-optimal cross-entropy. For this, we compute, for $d_{\phi}(\vx) = r_{\opt}(\vx) + h(\vx)$, with $h(\vx) =\sin(\omega (\vx))\sqrt{\varepsilon r_{\opt}(\vx)} $ :
\begin{align}
    \nabla_{\vx_t} \log(d_{\phi}(\vx)) = \nabla_{\vx_t} \log(r_{\opt}(\vx)) + \nabla_{\vx_t} \log\left( 1 + \sqrt{\frac{\varepsilon}{r_{\opt}(\vx)}} \sin(\omega (\vx)) \right)
\end{align}
Moreover, we have :
\begin{align}
    \nabla_{\vx_t} \log\left( 1 + \sqrt{\frac{\varepsilon}{r_{\opt}(\vx)}} \sin(\omega (\vx)) \right) = \frac{\sqrt{\varepsilon}\left(
        -\frac{1}{2}\nabla_{\vx_t} r_{\opt}(\vx)r_{\opt}^{-\frac{3}{2}}\sin(\omega (\vx)) + \sqrt{\frac{\varepsilon}{r_{\opt}(\vx)}} \omega \cos(\omega (\vx))
        \right)}{1 + \sqrt{\frac{\varepsilon}{r_{\opt}(\vx)}} \sin(\omega (\vx))}
\end{align}
Thus, the gain for a fixed timestep t is given by :
\begin{align}
    E_{\phi}^{t} & =\begin{aligned}
                        \mathbb{E}_{P_{t}}\left[
                        ||\nabla_{\vx_t} log (r_{\opt}(\vx)) - \nabla_{\vx_t} \log d_{\phi}(\vx) ||_{2}^{2}
                        \right]
                    \end{aligned}
    \\
                 & =\begin{aligned} \mathbb{E}_{P_{t}}\underbrace{\left|\left|\frac{\sqrt{\varepsilon}\left(
                            -\frac{1}{2}\nabla_{\vx_t} \left[r_{\opt}(\vx)\right]r_{\opt}^{-\frac{3}{2}}\sin(\omega (\vx)) + \sqrt{\frac{\varepsilon}{r_{\opt}(\vx)}} \omega \cos(\omega (\vx))
                            \right)}{1 + \sqrt{\frac{\varepsilon}{r_{\opt}(\vx)}} \sin(\omega (\vx))} \right|\right|_{2}^{2}}_{B_{\omega}(\vx)}
                    \end{aligned}
\end{align}
By setting $\omega$ to be very large, and using the following properties :
\begin{itemize}
    \item The set $\left\{ \vx \in\reals^d \middle\vert\cos(\omega \vx) = 0\right\})$ has a mass $0$ with respect to $P$.
    \item $\mathbb{E}_{P}[.] = P(r_{\opt}(\vx) = \infty)\mathbb{E}_{P}[. | r_{\opt}(\vx) = \infty] + P(r_{\opt}(\vx) < \infty) \mathbb{E}_{P}[. | r_{\opt}(\vx) < \infty]$
    \item $\mathbb{E}_{P}[. | r_{\opt}(\vx) = \infty] = 0$ and $P(r_{\opt}(\vx) < \infty) > 0$
\end{itemize}
We have that :
\begin{align}
    E_{\phi}^{t} & =
    P(r_{\opt}(\vx) = \infty) \mathbb{E}_{P_{t}}\left[B_{\omega}(\vx)
    \middle\vert r_{\opt}(\vx) = \infty\right] \notag                                                             \\
                 & \quad + P(r_{\opt}(\vx) < \infty) \mathbb{E}_{P_{t}}[B_{\omega}(\vx) | r_{\opt}(\vx) < \infty]
\end{align}
Thus, $E_{\phi}^{t} \to \infty$ as $\omega \to \infty$, which concludes our proof.
The effect of $\omega$ can be seen in Figures [\ref{fig:2Dplots},\ref{fig:2D-quivers},\ref{fig:2d_gradient}], where high values lose all the gradient information necessary for correcting the sampling process.
\begin{figure}[h]
    \centering
    \includegraphics[width=0.8\linewidth]{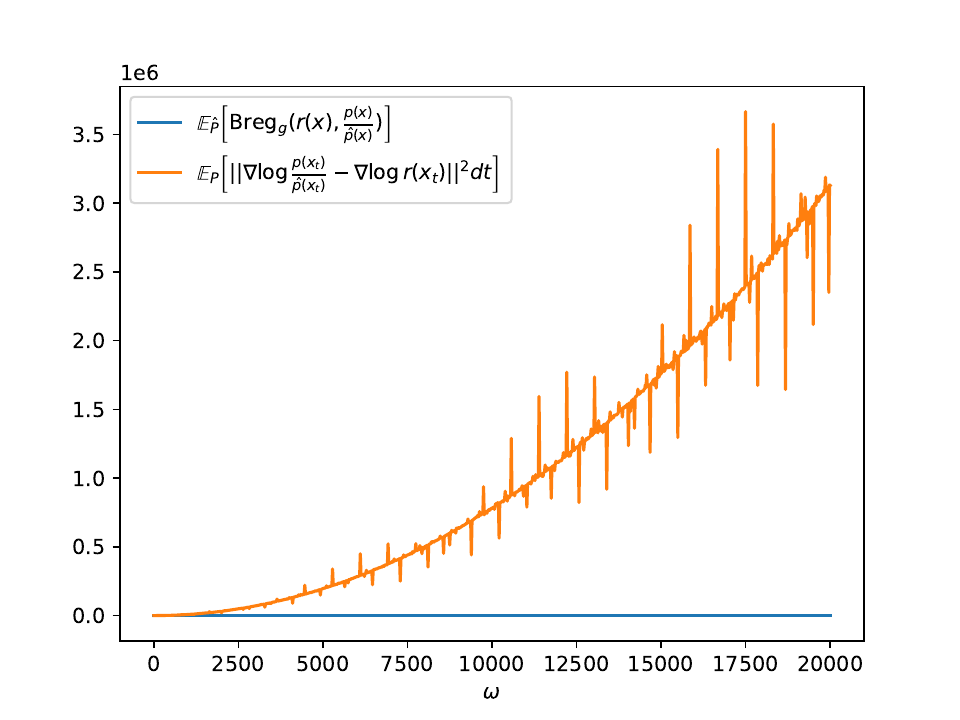}
    \caption{1-dimensional example : Density ratio between two Gaussians, with a discriminator having an $\varepsilon$-optimal cross-entropy and a gain that goes to infinity}
    \label{fig:1Dexple}
\end{figure}
\subsection{Proof of Theorem~\ref{thm:overfitting}}\label{app:sec:proofoverfitting}

\begin{proof}
    Define the measures $\tilde{\mu}=\min\left(P,\hat{P}\right),\mu_{P}=\max\left(0,P-\hat{P}\right)$
    and $\mu_{\hat{P}}=\max\left(0,\hat{P}-P\right)$. Observe that $P=\tilde{\mu}+\mu_{P}$,
    $\hat{P}=\tilde{\mu}+\mu_{\hat{P}}$ and that $\tilde{\mu}\left(\mathbb{R}\right)=1-TV$
    where $TV$ is the total variation distance between $P$ and $\hat{P}$.
    Define the distribution $\mu=\frac{\tilde{\mu}}{\tilde{\mu}(\mathbb{R})}$
    and let $f_{\mu}$ and $f_{\tilde{\mu}}$ be the density of $\mu$
    and $\tilde{\mu}$ with respect to Lebesgue measure. Note that because
    $\tilde{\mu}$ is not normalized, $f_{\tilde{\mu}}$ does not integrate
    to one.

    Let us build two sets $S$ and $S'$ iteratively using the following
    rejection sampling procedure: First, start with $S=S'=\emptyset$.
    Then, for each $i$, add $x_{i}$ (respectively $x'_{i}$) to the
    set $S$ (respectively $S'$) with probability $\frac{f_{\tilde{\mu}}(x_{i})}{p(x_{i})}$
    (respectively $\frac{f_{\tilde{\mu}}(x'_{i})}{\hat{p}(x'_{i})}$).
    For now, denote $M_{1}=\left|S\right|$ and $M_{2}=\left|S'\right|$.
    It is easy to check using standard properties of rejection sampling
    that $\mathbb{E}\left[M_{1}\right]=\mathbb{E}\left[M_{2}\right]=N\times(1-TV)$.
    Finally, take the largest of both sets $S$ and $S'$ and remove its
    last added elements until both sets have same cardinality $M=\min\left(M_{1},M_{2}\right)$.
    Note that in all three sets $S$,$S'$ and $S\cup S'$, examples are
    distributed i.i.d. from $\mu$.

    Let us define the mean square error:

    \begin{align*}
        \mathcal{L}_{MSE}^d(\phi) & =\mathbb{E}_{x\sim P}\left[\left(\nabla_{x}d^{\star}(x)-\nabla_{x}d_{\phi}(x)\right)^{2}\right]                                                                                                                \\
                                  & \ge\mathbb{E}\left[\left(\nabla_{x}d_{\phi}\right)^{2}\right]-2\mathbb{E}\left[\nabla_{x}d^{\star}.\nabla_{x}d_{\phi}\right]                                                                                   \\
                                  & \ge\mathbb{E}_{p}\left[\left(\nabla_{x}d_{\phi}\right)^{2}\right]-2\sqrt{\mathbb{E}_{p}\left[\left(\nabla_{x}d^{\star}\right)^{2}\right]}\sqrt{\mathbb{E}_{p}\left[\left(\nabla_{x}d_{\phi}\right)^{2}\right]}
    \end{align*}

    The last line follows from Cauchy-Schwartz inequality.

    We are interested into lower-bounding the \emph{expected} mean square
    error, noted $\mathbb{E}[ \mathcal{L}_{MSE}^d(\phi)]$, where the expectation is over the training
    set $\left\{ x_{1},x_{1}'\ldots\right\} $ on which the discriminator
    $d_{\phi}$ is trained. So we would like to bound it.

    To derive this bound, we will first lower bound $\mathbb{E}_{p}\left[\left(\nabla_{x}d_{\phi}\right)^{2}\right]$.

    First, consider an arbitrary interval $B\subset\mathbb{R}$ of size
    $\beta$ containing at least one point $z\in S$ and one point $z'\in S'$
    such that $z<z'$.

    Then we can write

    \begin{align*}
        \int_{\mathbb{R}}\left(\nabla d_{\phi}\right)^{2}dP(x) & \ge\int_{B}\left(\nabla d_{\phi}\right)^{2}dP(x)                                                                                     \\
                                                               & \ge\int_{B}\left(\nabla d_{\phi}\right)^{2}d\tilde{\mu}(x)=\int_{B}\left(\nabla d_{\phi}\right)^{2}f_{\tilde{\mu}}(x)dx              \\
                                                               & \ge\left(\inf_{x\in B}f_{\tilde{\mu}}(x)\right)\int_{B}\left(\nabla d_{\phi}\right)^{2}dx                                            \\
                                                               & \ge\left(\inf_{x\in B}f_{\tilde{\mu}}(x)\right)\int_{z}^{z'}\left(\nabla d_{\phi}\right)^{2}dx                                       \\
                                                               & =\left(\inf_{x\in B}f_{\tilde{\mu}}(x)\right)(z'-z)\int_{z}^{z'}\frac{\left(\nabla d_{\phi}\right)^{2}}{z'-z}dx                      \\
                                                               & \ge\left(\inf_{x\in B}f_{\tilde{\mu}}(x)\right)(z'-z)\left(\int_{z}^{z'}\frac{\nabla d_{\phi}}{z'-z}dx\right)^{2}\text{ (by Jensen)} \\
                                                               & =\frac{\inf_{x\in B}f_{\tilde{\mu}}(x)}{z'-z}\left(\int_{z}^{z'}\nabla d_{\phi}dx\right)^{2}                                         \\
                                                               & \ge\frac{\inf_{x\in B}f_{\tilde{\mu}}(x)}{\beta}\left(\int_{z}^{z'}\nabla d_{\phi}dx\right)^{2}
    \end{align*}

    Because both $p$ and $\hat{p}$ are L-Lipschitz, $f_{\tilde{\mu}}$
    has also this Lipschitz property. So for any $u\in B$ we have $f_{\tilde{\mu}}(u)-\inf_{x\in B}f_{\tilde{\mu}}(x)\le L\beta$,
    so $\inf_{x\in B}f_{\tilde{\mu}}(x)+L\beta\ge f_{\tilde{\mu}}(u)$,
    so $\inf_{x\in B}f_{\tilde{\mu}}(x)+L\beta\ge\frac{1}{\beta}\int_{B}f_{\tilde{\mu}}(x)dx=\frac{\tilde{\mu}\left(B\right)}{\beta}$.So
    we can write

    \[
        \int_{B}\nabla d_{\phi}^{2}dP(x)\ge\left(\frac{\tilde{\mu}\left(B\right)}{\beta^{2}}-L\right)\left(\int_{z}^{z'}\nabla d_{\phi}dx\right)^{2}
    \]

    Note that if $z>z'$ we would get by the same reasoning

    \[
        \int_{B}\nabla d_{\phi}^{2}dP(x)\ge\left(\frac{\tilde{\mu}\left(B\right)}{\beta^{2}}-L\right)\left(\int_{z'}^{z}\nabla d_{\phi}dx\right)^{2}
    \]

    To bound $\int_{z}^{z'}\nabla d_{\phi}dx$, recall that the cross-entropy
    for each $i$ is such that $\log\left(1+e^{-d_{\phi}(x_{i})}\right)\le\epsilon$
    and $\log\left(1+e^{d_{\phi}(x_{i}')}\right)\le\epsilon$. Thus,
    $1+e^{-d_{\phi}(x_{i})}\le e^{\epsilon}$ and $1+e^{d_{\phi}(x_{i}')}\le e^{\epsilon}$.
    So $d_{\phi}(x_{i})\ge\alpha$ and $d_{\phi}(x_{i}')\le-\alpha$
    where $\alpha=-\log\left(e^{\epsilon}-1\right)$. This also applies
    to $z$ and $z'$, so $\int_{\min(z,z')}^{\max(z,z')}\nabla d_{\phi}dx=\pm2\alpha$.
    Finally, we can summarize our findings by:

    \[
        \text{if both \ensuremath{S} and \ensuremath{S'} intersect with \ensuremath{B} then}\int_{B}\nabla d_{\phi}dP(x)\ge\left(\frac{\tilde{\mu}\left(B\right)}{\beta^{2}}-L\right)4\alpha^{2}
    \]

    Our goal now will be to show that there exists a small enough interval
    $B$ containing two points $z$ and $z'$ from $S$ and $S'$ with
    high probability.

    First, let us build an interval $A$ such that $\mu\left(A\right)\ge\frac{1}{2}$.
    By Chebyshev's inequality, we have that $\mu\left(\left\{ x:\left|x-\mathbb{E}_{\mu}x\right|\ge b\right\} \right)\le\frac{Var_{\mu}}{b^{2}}$
    for any $b$, so using the bound on the variance of lemma \ref{lem:variance_of_mu},
    we get that there exists a finite interval $A=\left[\mathbb{E}_{\mu}X-b,\mathbb{E}_{\mu}X+b\right]$
    where $b=\frac{\min\left(\sqrt{Var_{P}},\sqrt{Var_{\hat{P}}}\right)}{1-TV}$
    such that $\mu\left(A\right)\ge\frac{1}{2}$.

    Next, because $\mu(A)\ge\frac{1}{2}$, for any $0<\beta<2b$ there
    exists an interval $B\subset A$ of size $\beta$ such that $\mu\left(B\right)\ge\frac{\beta}{4b}$.
    Let us bound the probability that both $S$ and $S'$ intersect with
    $B$:

    \begin{align*}
        \mathbb{P}\left[B\cap S\neq\emptyset\wedge B\cap S'\neq\emptyset\right]= & 1-\mathbb{P}\left[B\cap S=\emptyset\vee B\cap S'=\emptyset\right]                                                                                       \\
        =                                                                        & 1-\mathbb{P}\left[B\cap S=\emptyset\right]-\mathbb{P}\left[B\cap S'=\emptyset\right]+\mathbb{P}\left[B\cap S'=\emptyset\wedge B\cap S'=\emptyset\right] \\
        =                                                                        & 1-\mu\left(\bar{B}\right)^{M}-\mu\left(\bar{B}\right)^{M}+\mu\left(\bar{B}\right)^{2M}                                                                  \\
        \ge                                                                      & 1-2\left(1-\mu\left(B\right)\right)^{M}\ge1-2\exp\left(-M\mu\left(B\right)\right)                                                                       \\
        \ge                                                                      & 1-2\exp\left(-\frac{M\beta}{4b}\right)
    \end{align*}

    For any non negative random variable $Z$, we know by the law of total
    expectation that $\mathbb{E}Z\ge\mathbb{E}\left[Z\mid condition\right]\mathbb{P}\left(condition\right)$,
    so in our case we have (here the expectation is over the dataset and
    the rejection sampling procedure):

    \begin{align*}
        \mathbb{E}\int_{B}\left(\nabla d_{\phi}\right)^{2}dP(x)\ge & \mathbb{E}\left[\int_{B}\left(\nabla d_{\phi}\right)^{2}dP(x)\mid B\cap S\neq\emptyset\wedge B\cap S'\neq\emptyset\right]      \\
                                                                   & \times\mathbb{P}\left[B\cap S\neq\emptyset\wedge B\cap S'\neq\emptyset\right]                                                  \\
        \ge                                                        & \left(\frac{\tilde{\mu}\left(B\right)}{\beta^{2}}-L\right)4\alpha^{2}\times\left(1-2\exp\left(-\frac{M\beta}{4b}\right)\right) \\
        \ge                                                        & \left(\frac{\mu\left(B\right)(1-TV)}{\beta^{2}}-L\right)4\alpha^{2}\times\left(1-2\exp\left(-\frac{M\beta}{4b}\right)\right)   \\
        \ge                                                        & \left(\frac{1-TV}{4b\beta}-L\right)4\alpha^{2}\times\left(1-2\exp\left(-M\frac{\beta}{4b}\right)\right)
    \end{align*}

    Choosing $\beta=\frac{b}{M}4\log4$ we get a $\left(1-2\exp\left(-M\frac{\beta}{4b}\right)\right)=1/2$
    and
    \begin{align*}
        \mathbb{E}\int_{B}\left(\nabla d_{\phi}\right)^{2}dP(x) & \ge\mathbb{E}\left(\frac{1-TV}{4b\beta}-L\right)2\alpha^{2}                                                         \\
                                                                & =\left(\frac{(1-TV)\mathbb{E}\left[M\right]}{8b^{2}\log4}-2L\right)\left(\log\left(e^{\epsilon}-1\right)\right)^{2}
    \end{align*}

    It is known that $\mathbb{E}M=\mathbb{E}\min\left(M_{1},M_{2}\right)\ge\frac{1}{2}\mathbb{E}M_{1}=\frac{N\times(1-TV)}{2}$.
    So we finally get\\
    \begin{align*}
        \mathbb{E}\int_{\infty}\left(\nabla d_{\phi}\right)^{2}dP(x) & \ge\left(N\frac{(1-TV)^{2}}{16b^{2}\log4}-2L\right)\left(\log\left(e^{\epsilon}-1\right)\right)^{2}                                \\
                                                                     & =\left(N\frac{(1-TV)^{4}}{16\log4\min\left(Var_{P},Var_{\hat{P}}\right)}-2L\right)\left(\log\left(e^{\epsilon}-1\right)\right)^{2}
    \end{align*}

    Finally, we can bound the expected MSE:

    \begin{align*}
        \mathbb{E}\left[\mathcal{L}_{MSE}^d(\phi) \right] & \ge\mathbb{E}_{p}\left[\left(\nabla_{x}d_{\phi}\right)^{2}\right]-2\sqrt{\mathbb{E}_{p}\left[\left(\nabla_{x}d^{\star}\right)^{2}\right]}\sqrt{\mathbb{E}_{p}\left[\left(\nabla_{x}d_{\phi}\right)^{2}\right]} \\
                                                          & \ge\mathbb{E}_{p}\left[\left(\nabla_{x}d_{\phi}\right)^{2}\right]-2\sqrt{\mathbb{E}_{p}\left[\left(\nabla_{x}d^{\star}\right)^{2}\right]}\sqrt{\mathbb{E}_{p}\left[\left(\nabla_{x}d_{\phi}\right)^{2}\right]} \\
    \end{align*}
\end{proof}
\begin{lemma}
    \label{lem:variance_of_mu}Let $P$ and $\hat{P}$ be two distributions
    over $\mathbb{R}$ and let $\tilde{\mu}=\min\left(P,\hat{P}\right)$.
    Then
    \[
        Var_{\mu}\le\frac{1}{2(1-TV)^{2}}\min\left(Var_{P},Var_{\hat{P}}\right)
    \]
\end{lemma}

\begin{proof}
    Define $\mu=\frac{\tilde{\mu}}{\tilde{\mu}(\mathbb{R})}$. As before,
    $\mu=\frac{\tilde{\mu}}{1-TV}$, so we have

    \begin{align*}
        Var_{\mu} & =\mathbb{E}_{X\sim\mu}\left[(X-\mathbb{E}X)^{2}\right]=\frac{1}{2}\mathbb{E}_{X,X'\sim\mu}\left[(X-X')^{2}\right] \\
                  & =\frac{1}{2}\int\int(x-x')^{2}d\mu(x)d\mu(x')                                                                     \\
                  & =\frac{1}{2(1-TV)^{2}}\int\int(x-x')^{2}d\tilde{\mu}(x)d\tilde{\mu}(x')                                           \\
                  & \le\frac{1}{2(1-TV)^{2}}\min\left(\int\int(x-x')^{2}dpdp,\int\int(x-x')^{2}d\hat{p}d\hat{p}\right)                \\
                  & =\frac{1}{2(1-TV)^{2}}\min\left(Var_{P},Var_{\hat{P}}\right)
    \end{align*}
\end{proof}

\subsection{Proof of Proposition~\ref{thm:mseloss}}\label{app:sec:proofmseloss}

\begin{proposition}
    Assume that $P$ and $\whP$ satisfy the assumptions detailed in Appendix~\ref{app:sec:assumptionssong}.
    Then, the following holds:
    \begin{equation}
        \argmin_\phi \calL_{\mathrm{SM}}^{d}(\phi) = \argmin_\phi \calL_{\mathrm{MSE}}^{d}(\phi).
    \end{equation}
\end{proposition}
First, we can show that:
\begin{align}
    \calL_{\mathrm{MSE}}^{d}(\phi) = & \int_0^T\lambda(t)   \E_{P_0, P_{t\vert\vx_0}}\left[\Vert \nabla_{\vx_t} \log p_t(\vx_t\vert\vx_0) - \nabla_{\vx_t} \log \whp_t(\vx_t) - \nabla_{\vx_t} \vd_\phi(\vx_t, t)\Vert_2^2 \right]\dt                              \\
    =                                & \begin{aligned}[t]
                                           \int_0^T & \lambda(t) \E_{P_t}\left[\frac{1}{2}\left\Vert\nabla_{\vx_t} \vd_\phi(\vx_t, t)\right\Vert_2^2\right] \dt                                                                          \\
                                                    & + \int_0^T\lambda(t) \E_{P_0, P_{t\vert\vx_0}}\left[\left\langle \nabla_{\vx_t}d(\vx_t, t), \nabla_{\vx_t} \log \frac{p_t(\vx_t\vert\vx_0)}{\whp_t(\vx_t)}\right\rangle\right] \dt \\
                                                    & + \int_0^T\lambda(t) \E_{P_0, P_{t\vert\vx_0}}\left[\frac{1}{2}\left\Vert\nabla_{\vx_t} \log \frac{p_t(\vx_t\vert\vx_0)}{\whp_t(\vx)}\right\Vert_2^2\right] \dt
                                       \end{aligned} \\
    =                                & \int_0^T\lambda(t) \E_{P_t}\left[\frac{1}{2}\left\Vert\nabla_{\vx_t} \vd_\phi(\vx_t, t)\right\Vert_2^2\right] \dt + J(\phi) + C_1,
\end{align}
where $C_1$ is a constant independent of $\phi$. And we can write $J$ as:
\begin{align}
    J(\phi) = & \int_0^T\lambda(t) \E_{P_0, P_{t\vert\vx_0}}\left[\left\langle \nabla_{\vx_t}d(\vx_t, t), \nabla_{\vx_t} \log \frac{p_t(\vx_t\vert\vx_0)}{\whp_t(\vx_t)}\right\rangle\right] \dt                            \\
    =         & \begin{aligned}[t]
                    \int_0^T & \lambda(t) E_{P_0, P_{t\vert\vx_0}}\left[\left\langle \nabla_{\vx_t}d(\vx_t, t), \nabla_{\vx_t} \log p_t(\vx_t\vert\vx_0)\right\rangle\right] \dt     \\
                             & - \int_0^T\lambda(t) E_{P_0, P_{t\vert\vx_0}}\left[\left\langle \nabla_{\vx_t}d(\vx_t, t), \nabla_{\vx_t} \log \whp_t(\vx_t)\right\rangle\right] \dt.
                \end{aligned}
\end{align}
Using the result of \citet{vincent_connection_2011} (Equation 17), we can show the term in the first integral can be rewritten as:
\begin{align}
    E_{P_0, P_{t\vert\vx_0}}\left[\left\langle \nabla_{\vx_t}d(\vx_t, t), \nabla_{\vx_t} \log p_t(\vx_t\vert\vx_0)\right\rangle\right] =  E_{P_t}\left[\left\langle \nabla_{\vx_t}d_\phi(\vx_t, t), \nabla_{\vx_t} \log p_t(\vx_t)\right\rangle\right].
\end{align}
Thus, we can rewrite $J$ as:
\begin{align}
    J(\phi) = & \begin{aligned}[t]
                    \int_0^T & \lambda(t) E_{P_t}\left[\left\langle \nabla_{\vx_t}d_\phi(\vx_t, t), \nabla_{\vx_t} \log p_t(\vx_t)\right\rangle\right] \dt                          \\
                             & - \int_0^T\lambda(t) E_{P_0, P_{t\vert\vx_0}}\left[\left\langle \nabla_{\vx_t}d(\vx_t, t), \nabla_{\vx_t} \log \whp_t(\vx_t)\right\rangle\right] \dt
                \end{aligned} \\
    =         & \int_0^T\lambda(t) E_{P_t}\left[\left\langle \nabla_{\vx_t}d_\phi(\vx_t, t),\nabla_{\vx_t} \log \frac{p_t(\vx_t)}{\whp(\vx_t)}\right\rangle\right] \dt.\label{eq:J}
\end{align}
And on the other side, we have:
\begin{align}
    \calL_{SM}^d(\phi) = & \int_0^T\lambda(t) \E_{P_t}\left[\Vert \nabla_{\vx_t} \vd_\phi(\vx_t, t)\Vert_2^2\right] \dt                                                                                                     \\
    =                    & \begin{aligned}[t]
                               \int_0^T & \lambda(t) \E_{P_t}\left[\frac{1}{2}\left\Vert\nabla_{\vx_t} \vd_\phi(\vx_t, t)\right\Vert_2^2\right] \dt                                               \\
                                        & + \int_0^T\lambda(t) \E_{P_t}\left[\left\langle \nabla_{\vx_t}d(\vx_t, t), \nabla_{\vx_t} \log \frac{p_t(\vx_t)}{\whp_t(\vx_t)}\right\rangle\right] \dt \\
                                        & + \int_0^T\lambda(t) \E_{P_t}\left[\frac{1}{2}\left\Vert\nabla_{\vx_t} \log \frac{p_t(\vx_t)}{\whp_t(\vx_t)}\right\Vert_2^2\right] \dt
                           \end{aligned} \\
    =                    & \int_0^T\lambda(t) \E_{P_t}\left[\frac{1}{2}\left\Vert\nabla_{\vx_t} \vd_\phi(\vx_t, t)\right\Vert_2^2\right] \dt + J(\phi) + C_2
\end{align}
using Equation~\eqref{eq:J} and $C_2$ is a constant independent of $\phi$. Thus, we have $\calL_{\mathrm{MSE}}^{d}(\phi) = \calL_{\mathrm{SM}}^{d}(\phi) + C_3$, where $C_3$ is a constant independent of $\phi$. Thus, the minimizer of $\calL_{\mathrm{MSE}}^{d}(\phi)$ is the same as the minimizer of $\calL_{\mathrm{SM}}^{d}(\phi)$.

\section{Additional plots}
\label{app:sec:additionalplots}
\begin{figure}[htbp!]
    \centering
    \includegraphics[width=\linewidth]{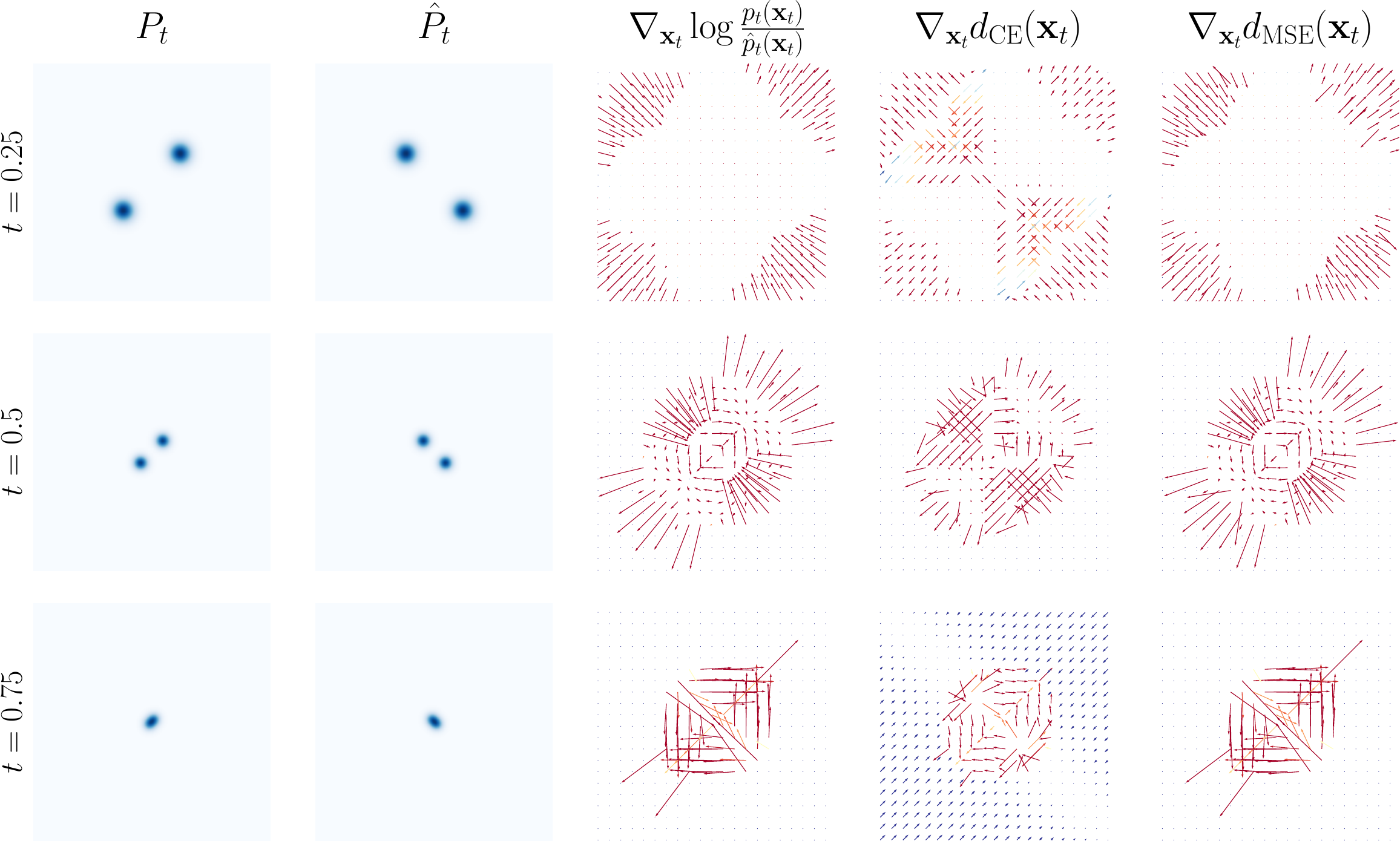}
    \caption{Gradients of the estimated density ratio. $d_{\mathrm{CE}}$ represents a discriminator with low cross entropy, and $d_{\mathrm{MSE}}$ represents a discriminator with low MSE}
    \label{fig:2D-quivers}
\end{figure}

\begin{figure*}[b!]
    \subfloat[EDM]{\includegraphics[width=0.32\textwidth]{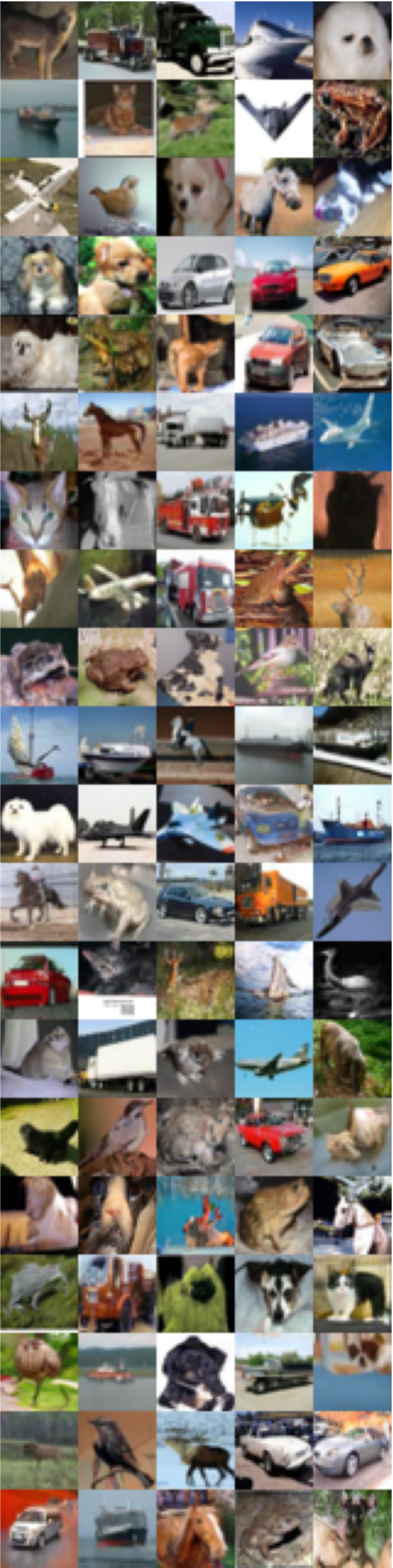}}\hfill
    \subfloat[EDM+DG]{\includegraphics[width=0.32\textwidth]{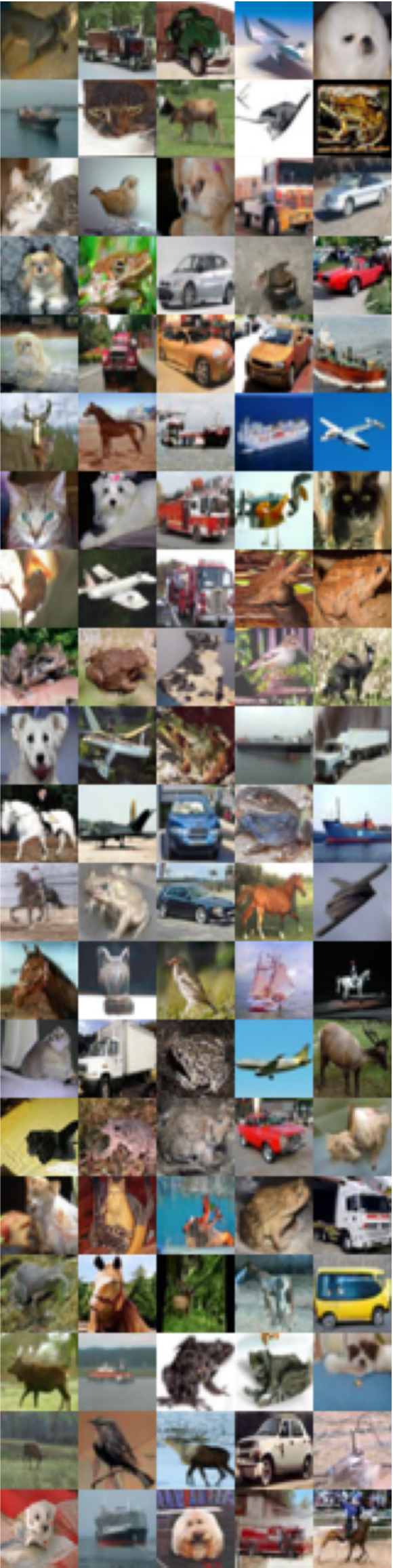}}\hfill
    \subfloat[EDM+Ours]{\includegraphics[width=0.32\textwidth]{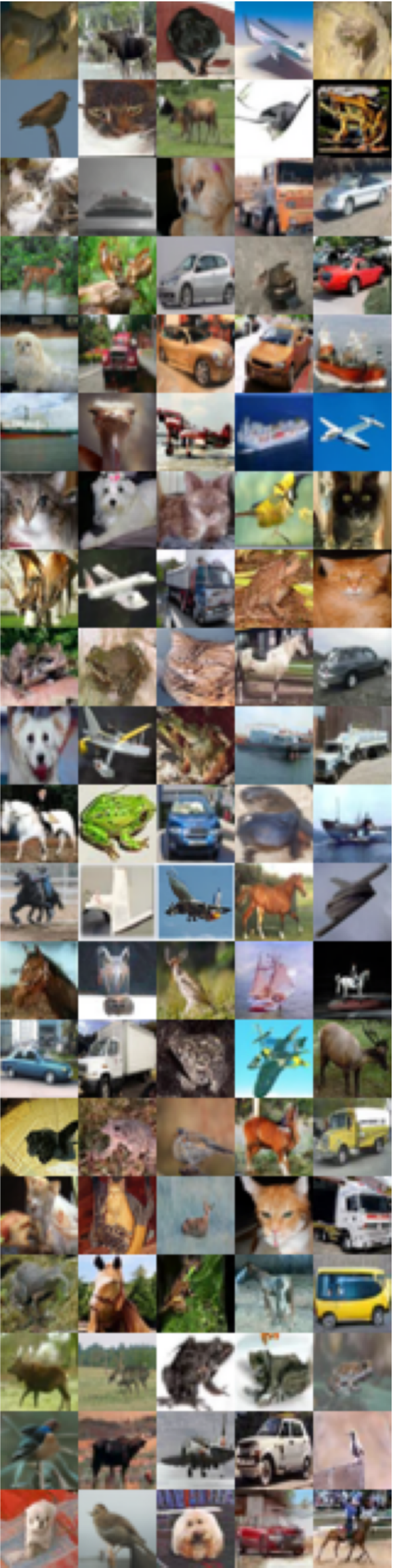}}
    \caption{Samples generated on the CIFAR-10 dataset.}
    \label{fig:cifarlarge}
\end{figure*}
\begin{figure*}[b!]
    \subfloat[EDM]{\includegraphics[width=0.32\textwidth]{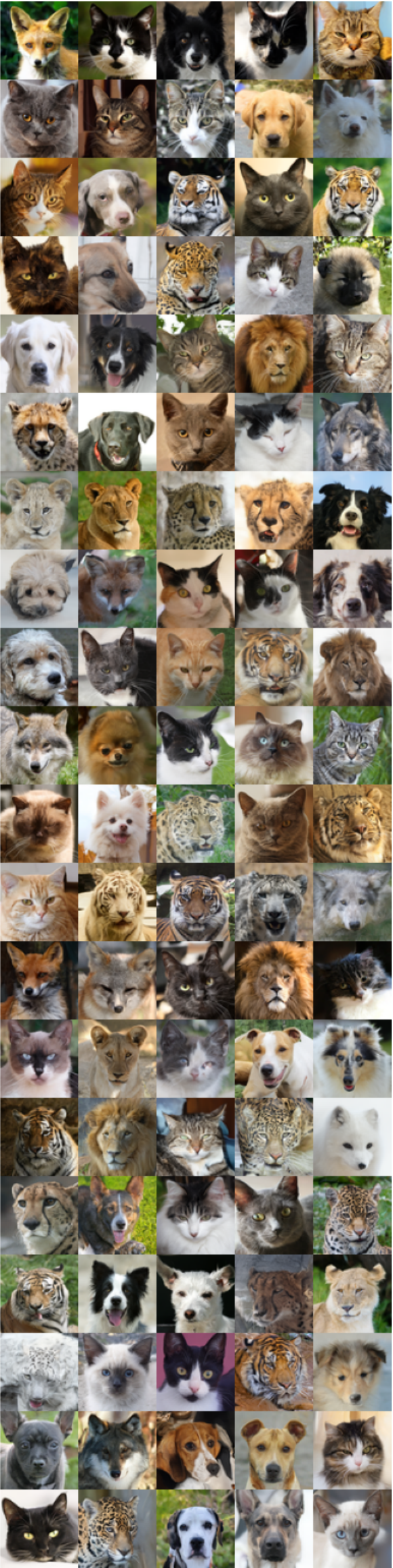}}\hfill
    \subfloat[EDM+DG]{\includegraphics[width=0.32\textwidth]{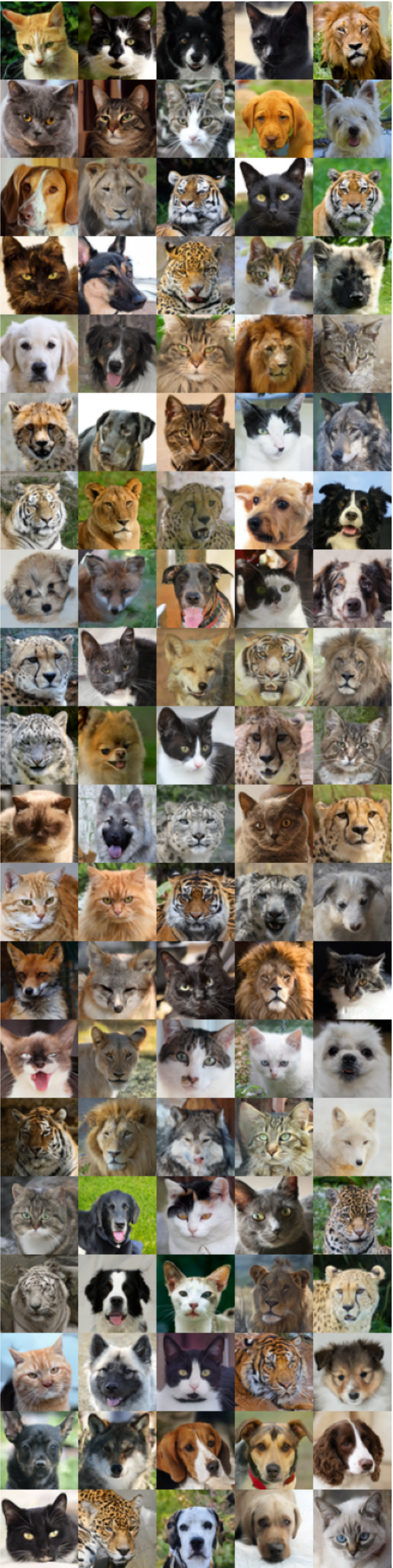}}\hfill
    \subfloat[EDM+Ours]{\includegraphics[width=0.32\textwidth]{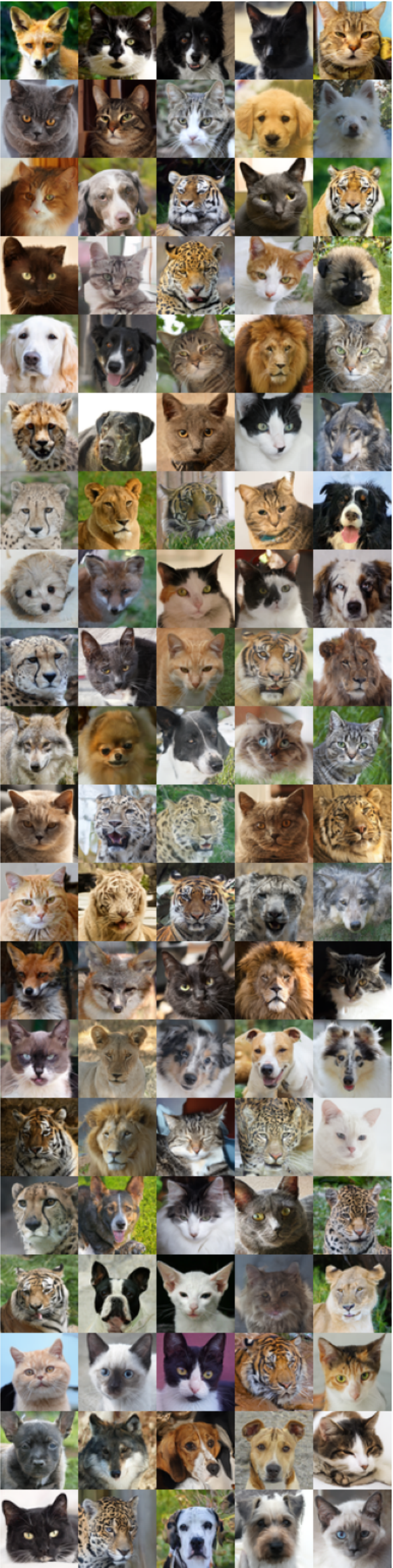}}
    \caption{Samples generated on the AFHQv2 dataset.}
    \label{fig:afhqlarge}
\end{figure*}

\begin{figure*}[b!]
    \subfloat[EDM]{\includegraphics[width=0.32\textwidth]{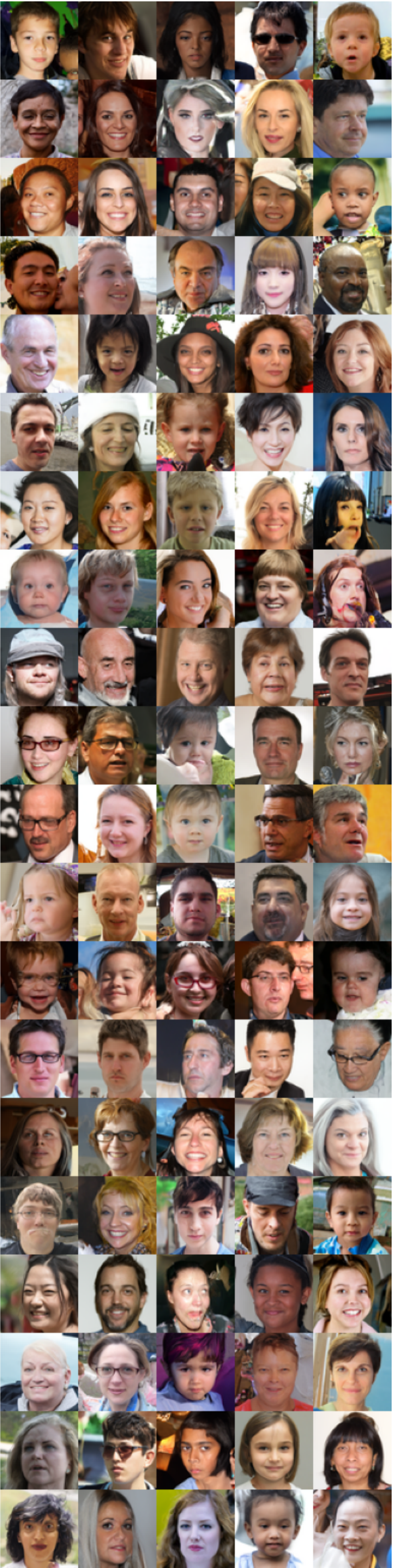}}\hfill
    \subfloat[EDM+DG]{\includegraphics[width=0.32\textwidth]{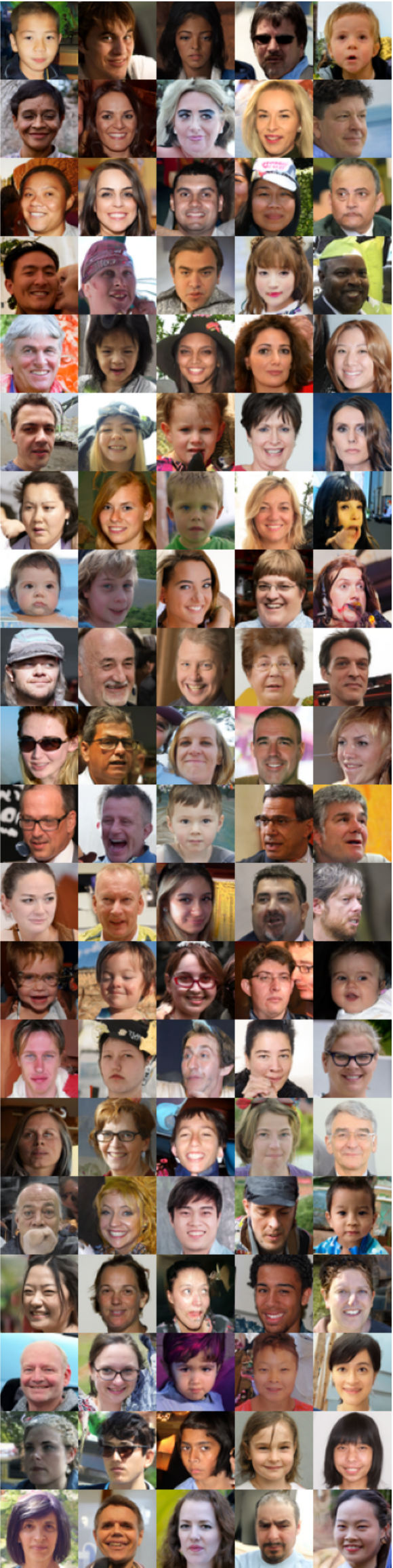}}\hfill
    \subfloat[EDM+Ours]{\includegraphics[width=0.32\textwidth]{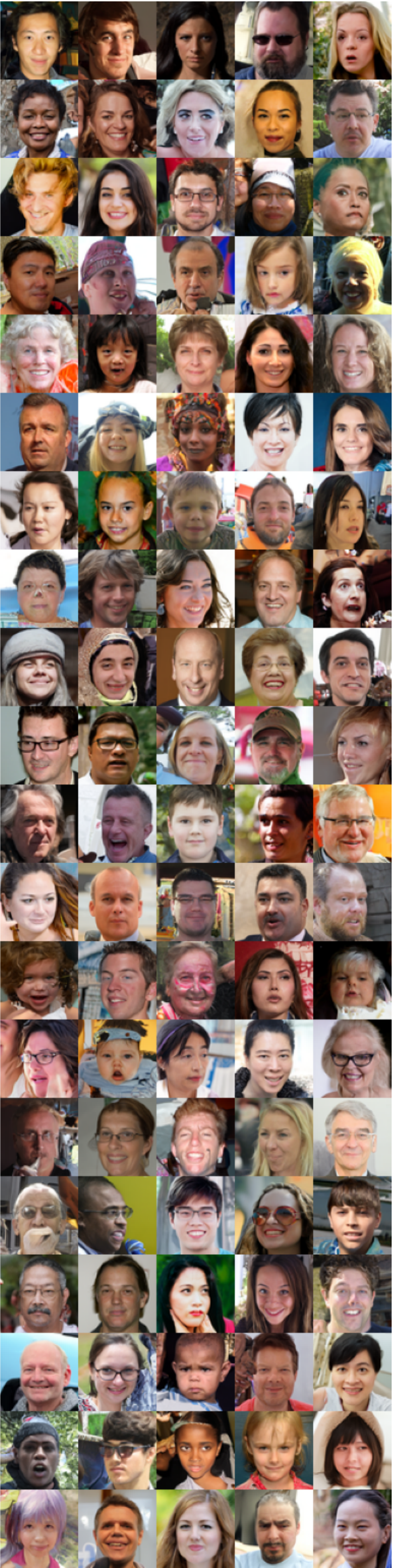}}
    \caption{Samples generated on the FFHQ dataset.}
    \label{fig:ffhq}
\end{figure*}

%




\end{document}